\documentclass[preprint,12pt,authoryear]{elsarticle}

\usepackage{latexsym}
\usepackage{amsthm}
\usepackage{booktabs}
\usepackage{enumitem}
\usepackage{graphicx}
\usepackage{color}
\usepackage{times}
\usepackage{soul}
\usepackage{url}
\usepackage[utf8]{inputenc}
\usepackage{graphicx}
\usepackage{amsmath}
\usepackage{amsthm}
\usepackage{booktabs}
\usepackage{algorithm}
\usepackage{algorithmic}
\usepackage[switch]{lineno}
\usepackage{amssymb}
\usepackage[table]{xcolor}

\usepackage{tikz}
\usetikzlibrary {arrows.meta}

\usepackage{balance}
\usepackage{subcaption}

\usepackage{ stmaryrd }

\newtheorem{theorem}{Theorem}

\newtheorem{proposition}[theorem]{Proposition}

\newtheorem{definition}{Definition}
\newtheorem{notation}{Notation}
\newtheorem{example}{Example}
\newtheorem*{proposition*}{Proposition}

\newcommand{\BibTeX}{B\kern-.05em{\sc i\kern-.025em b}\kern-.08em\TeX}

\renewcommand{\vec}[1]{\mathbf{#1}}

\newcommand{\accs}{\emph{\textup{Deg}}^{S}_{A}}
\newcommand{\acc}{\emph{\textup{Deg}}_{A}}
\newcommand{\accprime}{\emph{\textup{Deg}}_{A'}}
\newcommand{\accsprime}{\emph{\textup{Deg}}^{S}_{A'}}
\newcommand{\customaccs}[1]{\emph{\textup{Deg}}_{A}^{#1}}
\newcommand{\customacc}[1]{\emph{\textup{Deg}}_{#1}}
\newcommand{\customaccbis}[2]{\emph{\textup{Deg}}_{#1}^{#2}}

\newcommand{\mj}[1]{{#1}}

\newcounter{magicrownumbers}
\newcommand\rownumber{\stepcounter{magicrownumbers}\arabic{magicrownumbers}}

\journal{Artificial Intelligence}

\begin{document}

\begin{frontmatter}



\title{Aggregative Semantics for Quantitative Bipolar Argumentation Frameworks}


\author[SU]{Yann Munro} 
\author[SU]{Isabelle Bloch} 
\author[SU]{Marie-Jeanne Lesot} 

\affiliation[SU]{organization={Sorbonne Université, CNRS, LIP6},
            city={F-75005 Paris},
            country={France}}

\begin{abstract}
Formal argumentation is being used increasingly in artificial intelligence as an effective and understandable way to model potentially conflicting pieces of information, called arguments, and identify so-called acceptable arguments depending on a chosen semantics. This paper deals with the specific context of Quantitative Bipolar Argumentation Frameworks (QBAF), where arguments have intrinsic weights and can attack or support each other. In this context, we introduce a novel family of gradual semantics, called aggregative semantics. In order to deal with situations in which attackers and supporters do not play a symmetric role, and in contrast to modular semantics, we propose to aggregate attackers and supporters separately. This leads to a three-stage computation, which consists in computing a global weight for attackers and another for supporters, before aggregating these two values with the intrinsic weight of the argument. We discuss the properties that the three aggregation functions should satisfy depending on the context, as well as their relationships with the classical principles for gradual semantics. This discussion is supported by various simple examples, as well as a final example on which five hundred aggregative semantics are tested and compared, illustrating the range of possible behaviours with aggregative semantics. Decomposing the computation into three distinct and interpretable steps leads to a more parametrisable computation: it keeps the bipolarity one step further than what is done in the literature, and it leads to more understandable gradual semantics. 
\end{abstract}




\end{frontmatter}


\section{Introduction}

Formal argumentation is a theoretical framework used to model and reason about conflicting pieces of information. These pieces of information are called \emph{arguments} and a binary relation between them, called \emph{attack relation}, is used to represent their incompatibilities. Since the work of~\citet{dung1995acceptability}, enriched versions have been proposed to model more complex situations~\citep[see for example][]{bench2002value,cayrol2005gradual,li2011probabilistic,fazzinga2023incomplete}. In this paper, two of them are considered: bipolar argumentation frameworks~\citep{cayrol2005acceptability} and weighted argumentation frameworks~\citep{amgoud2017acceptability}. In bipolar argumentation, a second binary relation, independent of the attack relation and of an opposite nature, called \emph{support}, is introduced and represents a direct support of an argument. Thus, the bipolar framework combines positive support information and negative attack information, which can be asymmetric. Weighted argumentation frameworks add a weight on each argument, representing its intrinsic strength, modifying the impact of an argument on the arguments it attacks. In this paper, we consider the combination of these two formalisms called Quantitative Bipolar Argumentation Framework, QBAF~\citep[see][for a complete introduction]{baroni2019fine}.

To determine the status of each argument in this framework, two types of approaches have been proposed: extension-based semantics~\citep{cayrol2005acceptability} look for sets of collectively acceptable arguments, whereas gradual semantics~\citep{cayrol2005gradual} assign an acceptability numerical value to each argument. In this article we consider the latter: we propose a new gradual semantics family, called \emph{aggregative} semantics, which independently aggregates attacks and supports into two distinct values, and then combines these two values with the argument weight to compute its acceptability. Unlike existing gradual semantics, this approach breaks down the calculation of an argument acceptability into three distinct steps, allowing each of them to be constrained independently and to maintain bipolarity one step further. Moreover, by connecting to the field of aggregation operators, it opens the way to new semantics for argumentation. 

After the reminders provided in Section~\ref{sec:prelim}, the main contributions introduced in this paper, detailed in Sections~\ref{sec:agreg}, \ref{sec:discu_agreg}, \ref{sec:principles} and~\ref{sec:ex_comp} respectively, are:
\begin{itemize}
    \item the definition of a new family of gradual semantics, called aggregative semantics;
    \item an axiomatisation of the aggregative semantics along with a detailed and illustrated guide to choose appropriate functions;
    \item a formal comparison with the commonly used gradual semantics in the literature;
    \item an experimental evaluation of $515$ aggregative semantics on an illustrative example.
\end{itemize}

\section{Preliminaries}
\label{sec:prelim}

This section is divided into two subsections, each introducing a formalism used in this paper. In the first subsection, the basics of Quantitative Bipolar Argumentation Frameworks (QBAF) and the associated notion of gradual semantics are presented. Then, aggregation functions are defined, along with the properties that characterise them.

\subsection{Quantitative Bipolar Argumentation Frameworks and Gradual Semantics}

We first remind the formal definition of a QBAF as well as some useful notations~\citep[see][for more details]{baroni2019fine} before introducing the notion of gradual semantics~\citep{cayrol2005gradual}. We then present the main principles proposed in the literature to characterise them as well as the most commonly used gradual semantics. We finally discuss a specific principle called modularity, because it gives birth to a subclass of gradual semantics, called modular semantics~\citep{mossakowski2018modular}, whose definition is close to what we propose in Section~\ref{sec:agreg}.

\subsubsection{Definition of Quantitative Bipolar Argumentation Frameworks}

The concept of an Abstract Argumentation Framework~(AAF) has first been introduced by~\citet{dung1995acceptability} as a couple~${AF = (\mathcal{A},\mathcal{R})}$, where~$\mathcal{A}$ is a finite set of arguments and~$\mathcal{R}$ is a set of couples made of  an attacker and a target, defined as a binary relation on~$\mathcal{A} \times \mathcal{A}$, called the \emph{attack} relation. It has then been extended in order to be applicable to different scenarios. In this article, we consider the extension called Quantitative Bipolar Argumentation Framework:

\begin{definition}[Quantitative Bipolar Argumentation Framework~\citep{baroni2019fine}]

A Quantitative Bipolar Argumentation Framework (QBAF) is a quadruplet $A = (\mathcal{A},\mathcal{R},\mathcal{S},w)$ where: \begin{itemize}
    \item $(\mathcal{A},\mathcal{R})$ is an AAF;
    \item $\mathcal{S}$ is a set of couples made of a supporter and a target, defined as a binary relation on~$\mathcal{A} \times \mathcal{A}$, called the \emph{support} relation;
    \item $w: \mathcal{A} \rightarrow I_w$ is a function that assigns an intrinsic strength to each argument, taking values in a pre-ordered set of values~$I_w$, usually $I_w= [0,1]$.
\end{itemize}

 WAG denotes the set of QBAFs. Additionally, ac-WAG denotes the subset of acyclic QBAFs using $\mathcal{R} \cup \mathcal{S}$ as the set of edges for their associated graph.   
\end{definition}

Throughout this paper, we impose the following consistency assumption: if $\mathcal{R} \neq \emptyset$ and $\mathcal{S} \neq \emptyset$, then ${\mathcal{R} \cap \mathcal{S} = \emptyset}$, i.e. no argument both attacks and supports the same argument. We assume that such potential inconsistencies have been handled beforehand when building the QBAF.

Since $\mathcal{R}$ and $\mathcal{S}$ are binary relations defined over a finite set, a QBAF can be represented as a labelled directed graph in which the vertices represent the arguments, the labels represent their intrinsic strength, and the edges are of two types, representing the attack and support relations.  Figure~\ref{fig:QBAF} illustrates such a graph for the QBAF associated with Example~\ref{ex:toy_ex}, described later in this section.

In most articles in the QBAF literature, for any argument~$a$, the associated value $w(a)$  is referred to as its intrinsic strength~\citep{amgoud2018weighted,baroni2019fine}. However, its meaning can vary. For example, in a framework for modelling online debates, \citet{rago2017quantitative} associate this value with the number of positive votes in favour of an argument. \citet{potyka2021interpreting} proposes to use QBAFs to model neural networks and the intrinsic strength to represent the bias of each neuron. Another possibility is to consider that the intrinsic strength of an argument represents the a priori confidence in the source enunciating the argument~\citep{da2011changing}. 
In this paper, we do not consider any particular meaning to keep the discussion as general as possible.

Let us introduce some notations used in the sequel:
\begin{notation} Let $A=(\mathcal{A},\mathcal{R},\mathcal{S},w)$ and $A'=(\mathcal{A}',\mathcal{R}',\mathcal{S}',w')$ be two QBAFs, and $a \in \mathcal{A}$ an argument:

\begin{itemize}
	\item $\textup{Att}_a = \{b \in \mathcal{A} \mid (b,a) \in \mathcal{R} \}$ is the set of attackers of argument $a$;
	\item $\textup{Supp}_a = \{c \in \mathcal{A} \mid (c,a) \in\mathcal{S}\}$ is the set of supporters of argument $a$;
	\item  If ${\mathcal{A} \cap \mathcal{A}' = \emptyset}$, then the union of $A$ and $A'$, denoted $A'' = A\oplus A'$, is the QBAF~${(\mathcal{A}'',\mathcal{R}'',\mathcal{S}'',w'')}$ defined by: $\mathcal{A}'' = \mathcal{A} \cup \mathcal{A}'$, ${\mathcal{R}'' = \mathcal{R} \cup \mathcal{R}'}$, ${\mathcal{S}'' = \mathcal{S} \cup \mathcal{S}'}$, and $w''(a) = w(a)$ if $a \in \mathcal{A}$ and $w''(a) = w'(a)$ if $a \in \mathcal{A}'$;
	\item An isomorphism from $A$ into $A'$ is a bijection $f: \mathcal{A} \mapsto \mathcal{A}'$ which preserves relations and labels of the graphs: $\forall a,b \in \mathcal{A}, w(a) = w'(f(a)), (b,a) \in \mathcal{R}$ iff $(f(b),f(a)) \in \mathcal{R}'$, and $(b,a) \in \mathcal{S}$ iff $(f(b),f(a)) \in \mathcal{S}'$.
\end{itemize}
\end{notation}

We now introduce the illustrative toy example used throughout the paper.

\begin{example}\label{ex:toy_ex}
    Let us consider the QBAF $A = { (\mathcal{A},\mathcal{R},\mathcal{S},w)}$, with five formal arguments represented in Figure~\ref{fig:QBAF}: one central argument $a$ called the topic, two supporters for $a$ and two attackers against $a$.     
    Formally: 
    
    $\bullet$ $\mathcal{A} = \{a,b,c,d,e\}$ \hfill $\bullet$ $\mathcal{R} = \{(b,a),(e,a)\}$ \hfill $\bullet$ $\mathcal{S} = \{(c,a),(d,a)\}$\\and $w$ is shown in Figure~\ref{fig:QBAF}. 
\end{example}

\begin{figure}[t]
	\centering
	\begin{tikzpicture}[scale=0.9,transform shape, node distance={25mm}, main/.style = {draw, circle,minimum size=12mm}]
		\node[main] (1) {$a,0.5$};
		
		\node[main] (2)  at (-2.4,0.7) {$c,0.2$};
		\node[main] (3)  at (-2.4,-0.7) {$d,0.8$};
		\node[main] (4)  at (2.4,0.7) {$b,0.9$};
		\node[main] (5)  at (2.4,-0.7) {$e,0.1$};

		\draw [arrows = {-Latex[scale=1.5]}, dashed] (2) -- (1);
		\draw [arrows = {-Latex[scale=1.5]},dashed] (3) -- (1);
		\draw [arrows = {-Latex[scale=1.5]}] (4) -- (1);
		\draw [arrows = {-Latex[scale=1.5]}] (5) -- (1);

	\end{tikzpicture}
	\caption{Argumentation graph associated with Example~\ref{ex:toy_ex}. Plain lines represent the attack relation and dotted lines the support relation. The value beside each argument name is its intrinsic strength, in the range~$[0,1]$.}
	\label{fig:QBAF}
\end{figure}
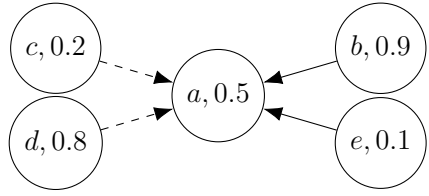

\subsubsection{Gradual Semantics}\label{sec:sem_grad}

Given a QBAF, a notion of acceptability degree of an argument is often used in order to infer an outcome, i.e. an acceptability status for each argument. This acceptability degree is computed using a function called a gradual semantics:

\begin{definition}[Gradual semantics]\label{def:sem_grad}
	
	A gradual semantics~$S$ is a function defined on WAG that outputs, for each QBAF $A$ in WAG, a function $\accs: \mathcal{A} \rightarrow I$, where $\accs(a)$ is the acceptability degree of argument $a$, taking values in the ordered set $I$, for instance $[0,1]$.	
\end{definition}

\noindent By abuse of language and for the rest of the paper, given a semantics $S$ and a WAG $A$, we also call semantics the induced acceptability function $\accs$, simply denoted $\acc$ if no ambiguity arises.

Several gradual semantics have been proposed in the literature. We present here the induced acceptability functions of the three most common ones:

$\bullet$ Discontinuity-Free Quantitative argumentation debate, DF-Quad~\citep{rago2016discontinuity}: 
\begin{align}
    \customaccs{DF}(a) &= w(a) - w(a) \max(0,-s) + (1-w(a)) \max(0,s) \nonumber \\    
    \textup{where } \displaystyle s &= \prod_{b \in \textup{Att}_a} (1-\customaccs{DF}(b)) - \prod_{c \in \textup{Supp}_a} (1-\customaccs{DF}(c)). \label{eq:df-quad}
\end{align}

$\bullet$ Euler-based, Ebs~\citep{amgoud2018weighted}: \begin{align}
    \displaystyle\customaccs{Ebs}(a) &= 1 - \frac{1-w(a)^2}{1+w(a)e^s} \nonumber\\
    \textup{where } \displaystyle s &= \sum_{c \in \textup{Supp}_a} \customaccs{Ebs}(c) - \sum_{b \in \textup{Att}_a} \customaccs{Ebs}(b). \label{eq:ebs}
\end{align}

$\bullet$ Quadratic Energy, QE~\citep{potyka2018continuous}: \begin{align}
    \displaystyle\customaccs{QE}(a) &= w(a) \left(1 - \frac{\max(0,-s)^2}{1 + \max(0,-s)^2}\right) + (1-w(a))\frac{\max(0,s)^2}{1 + \max(0,s)^2} \nonumber \\
    \textup{where }\displaystyle s &= \sum_{c \in \textup{Supp}_a} \customaccs{QE}(c) - \sum_{b \in \textup{Att}_a} \customaccs{QE}(b). \label{eq:qe}
\end{align} 

\paragraph{Remark} These three functions are defined recursively. Indeed, in order to compute the acceptability of an argument, the acceptability of other arguments in the graph must first be determined. In the case of cyclic graphs, convergence is not guaranteed a priori and requires a convergence study. Such a study has been performed for DF-Quad, Ebs and QE, concluding  that those semantics converge under certain constraints on the cyclic WAG~\citep{mossakowski2018modular,potyka2018continuous}. In the case of ac-WAG, this problem does not occur.

\begin{example}\label{ex:grad_sem}
    Considering the acyclic QBAF~$A$ represented in Figure~\ref{fig:QBAF}, these three common semantics respectively lead to the following acceptability degree for argument~$a$: \begin{itemize}
        \item $\customaccs{DF}(a) = 0.47$
        \item $\customaccs{Ebs}(a) = 0.5$
        \item $\customaccs{QE}(a) = 0.5$
    \end{itemize}

    Ebs and QE semantics return the same result because $s=0$, since the sum of the acceptability degrees of the attackers is equal to the sum of the acceptability degrees of the supporters. By contrast, DF-Quad favours the side of the strongest argument (b in this case), i.e. the attackers. Therefore, the acceptability degree of $a$ decreases as compared to its intrinsic weight.
\end{example}

In the literature~\citep{amgoud2018weighted}, arguments whose acceptability degree equals $0$ play a special role because they are considered worthless\footnote{We discuss those so-called worthless arguments later in Section~\ref{sec:discu_agreg}.}. Therefore, they are not taken into account in some properties. Consequently, we introduce the following notations:
\begin{notation} For a given semantics $S$, omitted in the notation below:

\begin{itemize}
	\item $\textup{sAtt}_a = \{b \in \textup{Att}_a \mid \acc(b) \neq 0 \}$ is the set of attackers of $a$ with non-zero acceptability degree;
	\item $\textup{sSupp}_a = \{c \in \textup{Supp}_a \mid \acc(c) \neq 0\}$ is the set of supporters of $a$ with non-zero acceptability degree.
\end{itemize}
\end{notation}

\subsubsection{Principles for Gradual Semantics}\label{sec:ax_aaf}

Definition~\ref{def:sem_grad} is very general and does not impose any constraint on the form that $\acc$ should take. An axiomatic approach has been proposed to guide the choice of this function according to a desired behaviour~\citep{amgoud2018weighted}, and was further developed by~\citet{mossakowski2018modular} and \citet{potyka2019extending}. In particular, \citet{amgoud2018weighted} formulate twelve general principles, named axioms in the seminal paper. We provide a brief description of each of the twelve original principles below, reproducing their formal definition and interpretation.

Let $A=(\mathcal{A},\mathcal{R},\mathcal{S},w),A'=(\mathcal{A}',\mathcal{R}',\mathcal{S}',w')$ be two QBAFs with ${I_w = I = [0,1]}$:

\begin{description}
	\item[(A1) --] Anonymity: The acceptability degree of an argument is invariant under isomorphism, expressing the indistinguishability of isomorphic QBAFs. \\Formally, given two isomorphic QBAFs $A$ and $A'$ and an isomorphism $f$ from  $A \textup{ to } A'$: $$\forall a \in \mathcal{A}, \acc(a) = \accprime(f(a))$$
	\item[(A2) --] Independence: The acceptability degree of an argument only depends on the vertices it is connected to. In other words, it is possible to handle each connected component of a graph independently. \\Formally, $\textup{if } \mathcal{A} \cap \mathcal{A}' = \emptyset \textup{, and then } \mathcal{R} \cap \mathcal{R}' = \mathcal{S} \cap \mathcal{S}' = \emptyset:$ $$\forall a\in \mathcal{A}, \acc(a) = \customacc{{A}\oplus A'}(a)$$
	\item[(A3) --] Directionality: The acceptability degree of an argument only depends on its ancestors, i.e. its attackers (direct or indirect) and supporters (direct or indirect). \\Formally, if $\mathcal{A} = \mathcal{A'}, w = w', \mathcal{R} \subseteq \mathcal{R}'$ and $\mathcal{S} \subseteq \mathcal{S}'$ then: \begin{center}
		$\forall a,b,x \in \mathcal{A}, \textup{ if }\mathcal{R}' \cup \mathcal{S}' = \mathcal{R} \cup\mathcal{S}\cup \{(a,b)\} \textup{ and there is no path from } b \textup{ to } x,$ then $\acc(x) = \accprime(x)$
	\end{center} 
	\item[(A4) --] Equivalence: The acceptability degree of an argument only depends on its intrinsic strength and the acceptability degree of its attackers and supporters. \\Formally, given $a,b \in \mathcal{A}$ two arguments such that: \begin{itemize}
		\item $w(a) = w(b)$ 
		\item $\exists f: \textup{Att}_a \to \textup{Att}_b$ bijective such that $\forall x \in \textup{Att}_a, \acc(x) = \acc(f(x))$
		\item $\exists f': \textup{Supp}_a \to \textup{Supp}_b$ bijective such that: $\forall y \in \textup{Supp}_a, {\acc(y) = \acc(f'(y))}$
	\end{itemize}
\smallskip
then $\acc(a) = \acc(b)$
	\item[(A5) --] Stability: The acceptability degree of an unattacked and unsupported argument is equal to its intrinsic strength. \\Formally, $\forall a \in \mathcal{A}$: \begin{center}
		if $\textup{Att}_a = \textup{Supp}_a = \emptyset$, then $\acc(a) = w(a)$
	\end{center}
	\item[(A6) --] Neutrality: Arguments whose acceptability degree equals $0$ have no impact on the acceptability degree of the arguments they are attacking or supporting.\\Formally, given  $a,b,x \in \mathcal{A}$ three arguments such that:\begin{itemize}
		\item $w(a) = w(b)$
		\item $\acc(x) = 0$ 
		\item $\textup{Att}_{a} \subseteq \textup{Att}_b, \textup{Supp}_{a} \subseteq \textup{Supp}_b$
		\item $\textup{Att}_b \cup \textup{Supp}_b = \textup{Att}_a \cup \textup{Supp}_a \cup \{x\}$
	\end{itemize}
\smallskip
then $\acc(a) = \acc(b)$ 
	\item[(A7) --] Monotony: An argument $b$ with the same intrinsic strength as another argument $a$ but having at least the same set of attackers and at most the same set of supporters has an acceptability degree less than or equal to that of $a$. \\Formally, let $a,b \in \mathcal{A}$ such that $w(a) = w(b)$. \\If $\textup{Att}_a \subseteq \textup{Att}_b$ and ${\textup{Supp}_b \subseteq \textup{Supp}_a}$, then: \begin{itemize}
		\item $\acc(a) \geq \acc(b)$ \hfill (Monotony)
    \end{itemize}
    If additionally $\acc(a)>0$ and $ \textup{sAtt}_a \subsetneq \textup{sAtt}_b$, \\ or if $\acc(b)<1$ and ${\textup{sSupp}_b \subsetneq \textup{sSupp}_a}$, then:\begin{itemize}
        \item $\acc(a)> \acc(b)$ \hfill (Strict monotony)
    \end{itemize}
	\item[(A8) --] Reinforcement: Consider that two arguments, $a$ and $b$, that have the same intrinsic strength and are attacked by the same set of arguments, and supported by the same set of arguments with only one attacker and one supporter differing between them. If this unique attacker of $a$ is less acceptable than this unique attacker of $b$, and if this unique supporter of $a$ is more acceptable than this unique supporter of $b$, then argument $a$ is at least as acceptable as argument $b$. \\
Formally, let $a,b \in \mathcal{A}$, $C, C' \subseteq \mathcal{A}$, and $x,x',y,y' \in \mathcal{A} \setminus (C \cup C')$ such that:
\begin{itemize}
	\item $w(a) = w(b)$
	\item $\textup{Att}_a = C \cup \{x\}$ and $\textup{Att}_b = C \cup \{y\}$
	\item $\textup{Supp}_a = C' \cup \{x'\}$ and $\textup{Supp}_b = C' \cup \{y'\}$
	\item $\acc(x) \leq \acc(y)$ and $\acc(x') \geq \acc(y')$
\end{itemize}
then:
\begin{itemize}
	\item $\acc(a) \geq \acc(b)$ \\ \phantom{a} \hfill (Reinforcement)
	\item If additionally $\acc(a) > 0$ and $\acc(x) < \acc(y)$, 
    \\or if ${\acc(b) < 1}$ and $\acc(x') > \acc(y')$, 
    \\then $\acc(a) > \acc(b)$ \\ \phantom{a} \hfill (Strict Reinforcement)
\end{itemize}
	\item[(A9) --] Resilience: An argument whose intrinsic strength is neither maximal nor minimal cannot have a maximal or minimal acceptability degree. \\ Formally, $\forall a \in \mathcal{A}$:
\begin{center}
	If $0 < w(a) < 1$, then $0 < \acc(a) < 1$
\end{center}

\item[(A10) --] Franklin’s Principle: Adding an attacker and a supporter with the same acceptability degree to an argument decreases or maintains the acceptability of that argument. \\Formally, let $a,b,x,y \in \mathcal{A}$ such that:
\begin{itemize}
	\item $w(a) = w(b)$
	\item $\textup{Att}_a = \textup{Att}_b \cup \{x\}$ and $\textup{Supp}_a = \textup{Supp}_b \cup \{y\}$
	\item $\acc(x) = \acc(y)$
\end{itemize}
then:
\begin{itemize}
	\item $\acc(a) \leq \acc(b)$ \hfill (Franklin’s Principle)
	\item $\acc(a) = \acc(b)$ \hfill (Strict Franklin Principle)
\end{itemize}

\item[(A11) --] Weakening: If the attackers of an argument are of ``better quality'' than the supporters of this argument, then its acceptability is less than its intrinsic strength. More precisely, if each supporter of an argument $a$ can be matched with a distinct attacker that is at least as acceptable, and if some of these attackers are more acceptable than their matched supporters or if there are additional attackers that are not matched, then the acceptability of argument~$a$ is lower than its intrinsic strength. \\Formally, let $a \in \mathcal{A}$. If $w(a) > 0$ and there exists an injective function $f$ from $\textup{Supp}_a$ to $\textup{Att}_a$ such that:
\begin{itemize}
	\item $\forall x \in \textup{Supp}_a,\ \acc(x) \leq \acc(f(x))$
	\item $\textup{sAtt}_a \setminus \{f(x)\ |\ x \in \textup{Supp}_a\} \neq \emptyset$
    \\or $\exists x \in \textup{Supp}_a$ such that $\acc(x) < \acc(f(x))$
\end{itemize}
then $\acc(a) < w(a)$.

\item[(A12) --] Strengthening: If the supporters of an argument are of ``better quality'' (in the same sense as in \textbf{A11}) than its attackers, then its acceptability is greater than its intrinsic strength. \\Formally, let $a \in \mathcal{A}$. If $w(a) > 0$ and there exists an injective function $f$ from $\textup{Att}_a$ to $\textup{Supp}_a$ such that: \begin{itemize}
		\item $\forall x \in \textup{Att}_a$, $\acc(x) \leq \acc(f(x))$
		\item $\textup{sSupp}_a \setminus  \{f(x) | x \in \textup{Att}_a \} \neq \emptyset $ 
        \\or  $\exists x \in \textup{Att}_a$ such that $\acc(x) < \acc(f(x))$
	\end{itemize}
	\smallskip
	then $\acc(a) > w(a)$.
	
\end{description}

In this paper, we study additional principles of gradual semantics and formally compare them with these original twelve axioms. 

\subsubsection{Modular semantics}\label{sec:modular_sem}

\citet{mossakowski2018modular} define additional principles for gradual semantics. Among them, \emph{modularity} specifies the form that the acceptability function induced by the semantics should take. Consequently, this principle gives birth to a subclass of gradual semantics called \emph{modular semantics}. Rather than using an argumentative formalism, Mossakowski and Neuhaus' work is based on a matrix representation of QBAFs. Later in the literature, the notion of modular semantics has been reformulated, not as a principle, but directly as a family of gradual semantics and with an argumentation-based formalisation~\citep[see for example the work by][]{potyka2019extending,degradual,kampik2024change,potyka2024balancing}. In this section we present these two definitions of a modular semantics.

We first remind the original definition.

\begin{definition}[Modular semantics~\citep{mossakowski2018modular}]\label{def:modular1}
A modular semantics is a function $S$ that maps a QBAF $A = (\mathcal{A},\mathcal{R},\mathcal{S},w)$ to an acceptability function $\accs$, taking values in any connected subset of $\mathbb{R}$, that can be written, for any $a\in \mathcal{A}$, as: $${\accs(a) = i( \alpha(\mathbf{g}_a,\mathbf{s}),w(a))}$$

where: \begin{itemize}
    \item $\mathbf{g}_a$ is a vector expressing the relations to $a$, i.e. a vector of $-1,0,1$ where $1$ stands for an attacker of $a$, $-1$ for a supporter $a$ and $0$ for neither,
    \item $\mathbf{s}$ is the vector of the acceptability degrees of all arguments,
    \item $\alpha: \{-1,0,1\}^{|\mathcal{A}|} \times \mathbb{R}^{|\mathcal{A}|} \rightarrow \mathbb{R}$ is an aggregation function,
    \item $i:\mathbb{R}^2 \rightarrow \mathbb{R}$ is a so-called influence function.
\end{itemize} 
\end{definition} 

Definition~\ref{def:modular1} is recursive as it uses the acceptability status of the other arguments of the graphs. There is no guarantee of convergence in case of cyclic graphs, and the convergence depends on the choice of $\alpha$ and $i$ ~\citep{mossakowski2018modular,potyka2018continuous}. For this reason, most applications in the literature are limited to acyclic graphs~\citep{rago2017quantitative,amgoud2018weighted,kampik2024change} for which convergence is guaranteed, and modular semantics are often defined as partial functions over ac-WAG, i.e. functions restricted to acyclic graphs.

Definition~\ref{def:modular1} has been reformulated by~\citet{kampik2024change} using this notion of partial function as well as usual notations in argumentation, as reminded in Definition~\ref{def:modular2}.

\begin{definition}[Modular semantics~\citep{kampik2024change}]\label{def:modular2}
Using the curly bracket notation~$\{\}$ to represent multisets, a modular semantics is a partial function $S$ that maps an acyclic QBAF $A = (\mathcal{A},\mathcal{R},\mathcal{S},w)$ to an acceptability function $\accs$ that can be written, for any $a\in \mathcal{A}$, as: 
$$\accs(a) = i_{w(a)}(\alpha(\{\accs(x) \mid x \in \textup{Att}_a\},\{\accs(y) \mid y \in \textup{Supp}_a\}))$$

\noindent where: \begin{itemize}
    \item $\alpha: \mathbb{R}^{\mid \textup{Att}_a\mid} \times \mathbb{R}^{\mid \textup{Supp}_a \mid}\rightarrow \mathbb{R}$ is an aggregation function,
    \item $i_{w(a)}:\mathbb{R} \rightarrow \mathbb{R}$ is a so-called parametrised influence function.
\end{itemize} 
\end{definition}


To determine the acceptability degree of an argument, a modular semantics separates the acceptability of the other arguments (whether attacking or supporting) from the intrinsic weight of the concerned argument, breaking the computation down into two distinct steps. This allows each step to be constrained independently, resulting in better customisation of the semantics depending on the context. It also increases the transparency compared to a black-box function and is a more natural procedure. This is illustrated by the fact that most of the gradual semantics used in the literature before modular semantics were defined were already modular. Examples include DF-Quad and Ebs (QE is also modular, but it was defined later). However, this decomposition does not necessarily disentangle attackers and supporters during the aggregation steps.

\subsection{Aggregation Functions}\label{sec:aggreg_func}

The second central element of the contributions of this paper is the notion of aggregation function~\citep{grabisch2009aggregation}. In this section, we first give the definition of an aggregation function used in this paper, before introducing several properties that have been studied for those functions. Then we provide examples of classes of such functions.

\begin{definition}[Aggregation function]
    An aggregation function is a function that combines a multiset of values into a single value.
\end{definition} 

\paragraph{Remark} 
In the literature, some authors distinguish between aggregation and fusion operators depending on the nature of the values to aggregate. However, in the rest of the paper, we use the expression ``aggregation function'' to refer to functions that combine multiple variables of possibly different natures into a single value.

In the rest of the paper, the discussed aggregation functions take as inputs values in $[0,1]$ and output a value in the same interval, except if otherwise specified.

Aggregation functions can take many forms and have been applied in numerous different contexts~\citep[see for example][]{dubois1985review,yager1998full,torra2007modeling,grabisch2009aggregation,bobillo2013aggregation}. Therefore, just like for gradual semantics, many properties have been proposed in the literature to characterise their behaviour \citep[see for example the reviews by][]{dubois1985review,SMC-96,konieczny2004da2}. The main ones are as follows, where $\varphi$ denotes an aggregation operator that can take any number of arguments. Variables written in bold and indexed by $n$, such as $\mathbf{x}_n= (x_1,\dotsc x_n)$,  are elements of $[0,1]^n$, and the other variables are elements of~$[0, 1]$:
\begin{description}
	\item[(P1) --] Boundary conditions: $\varphi(0,\dotsc, 0) = 0$ and $\varphi(1, \dotsc, 1) = 1$.
	\item[(P2) --] Monotony (non decreasingness): 
			$\forall \mathbf{x_n}, \mathbf{x'_n} \in [0,1]^n$ if $\exists i \in \llbracket1,n \rrbracket$ such that $\forall j \in  \llbracket1,n \rrbracket \setminus\{i\}$, $x_j = x'_j$ and $x_i\leq x'_i$, then
			$\varphi(\mathbf{x}_n) \leq \varphi(\mathbf{x}'_n)$.

            This property is usually imposed in most definitions of aggregation functions. However, in Section~\ref{sec:discu_agreg}, we discuss this property, and consequently we do not impose it in the following. 
    \item[(P3) --] Continuity:  
			$\lim\limits_{\mathbf{x}'_n \to \mathbf{x}_n} \varphi(\mathbf{x}'_n) = \varphi(\mathbf{x}_n)$.
	\item[(P4) --] Commutativity: for any permutation $\sigma$ on $\{1,\dotsc,n\}$, $\varphi(\mathbf{x}_n) = \varphi(x_{\sigma(1)},\dotsc,x_{\sigma({n})})$.
	\item[(P5) --] Idempotence: 
			$\varphi(x,\dotsc,x) = x$.
    \item[(P6) --] Associativity:
			$ \varphi(\varphi(\mathbf{x}_n),\varphi (\mathbf{y}_m)) = \varphi(\mathbf{x}_n,\mathbf{y}_m)$.
	\item[(P7) --] Weakening: $\varphi(\textbf{x}_n) \leq \min(\textbf{x}_n)$, where $\min(\textbf{x}_n) = \min(x_1, \dotsc x_n)$.
	\item[(P8) --] Reinforcement: $\varphi(\textbf{x}_n) \geq \max(\textbf{x}_n)$, where $\max(\textbf{x}_n) = \max(x_1, \dotsc x_n)$.
    \item[(P9) --] Neutral element $e_1$: 
    $\forall i \in \llbracket1,n \rrbracket$, if $x_i = e_1$, 
    \\then $\varphi(x_1,\dotsc, x_n) = \varphi(x_1,\dotsc, x_{i-1}, x_{i+1},\dotsc, x_n)$.
    \item[(P10) --] Null element $e_0$: if $\exists i \in \llbracket1,n \rrbracket$ such that $x_i = e_0$, then ${\varphi(x_1,\dotsc, x_n) = e_0}$.
	\item[(P11) --] Composition: $\forall \mathbf{x_n}, \mathbf{y_m} \in [0,1]^n \times [0,1]^m$ and $\forall z \in [0,1]$, 
    \\if $\varphi(\mathbf{x}_n) \leq \varphi(\mathbf{y}_m)$, then $\varphi(\mathbf{x}_n,z) \leq \varphi(\mathbf{y}_m,z)$, for any position of $z$.
	\item[(P12) --] Decomposition: $\forall \mathbf{x_n}, \mathbf{y_m} \in [0,1]^n \times [0,1]^m$ and $\forall z \in [0,1]$, \\if $\varphi(\mathbf{x}_n,z) < \varphi(\mathbf{y}_m,z)$, then $\varphi(\mathbf{x}_n) \leq \varphi(\mathbf{y}_m)$, for any position of $z$.
    \end{description}
The interpretation of these properties is straightforward. Composition and decomposition describe the behaviour of $\varphi$ when a variable is added or suppressed from the list of arguments.

\paragraph{Remark} 
Note that the names of some properties of aggregation functions are also used for properties that may differ in argumentation~(see Section~\ref{sec:ax_aaf}). In order to distinguish between them, in the following, we use the word \emph{postulate} to refer to properties defined for aggregation functions, and we keep the term \emph{principle} for properties derived from argumentation. For example, we distinguish the monotony postulate, denoted \textbf{P2}, which gives the direction of variation of an aggregation function, from the monotony principle, denoted \textbf{A7}, described in Section~\ref{sec:ax_aaf}.

In the huge spectrum of aggregation operators, typical examples are t-norms that are conjunctive, t-conorms, that are disjunctive, and compromise operators. Let us briefly recall their definitions and properties~\citep{dubois1985review}:

\renewcommand{\arraystretch}{1.3}
\begin{table}[t]
\centering
\resizebox{\linewidth}{!}{\begin{tabular}{|c|c|c|c|}
\hline
Name             & Symbol & \multicolumn{1}{c|}{Mathematical expression} &  Family \\ \hline
Arithmetic mean  &  $\textup{avg}_\textup{am}$ & $\textup{avg}_\textup{am}(\mathbf{x_n}) = \displaystyle \frac{1}{n} \sum_{i = 1}^n x_i$        &     Compromise      \\ \hline
Geometric mean   &  $\textup{avg}_\textup{gm}$ & $\textup{avg}_\textup{gm}(\mathbf{x_n}) = \displaystyle \left( \prod_{i = 1}^n x_i \right)^{1/n}$          &                                         Compromise      \\ \hline
Product          & $\top_\textup{prod}$ &   $\top_\textup{prod}(\mathbf{x_n}) = \displaystyle  \prod_{i = 1}^{\phantom{I} n \phantom{I}} x_i$     &                                              t-norm             \\ \hline
Algebraic sum    & $\bot_\textup{prod}$&   $\bot_\textup{prod}(\mathbf{x_n}) = 1 - \displaystyle \prod_{i = 1}^n (1-x_i)$     &                                              t-conorm           \\ \hline
Minimum           &  $\min$ &    $\min(\mathbf{x_n})$    &                                               t-norm             \\ \hline
Maximum            & $\max$ &    $\max(\mathbf{x_n})$   &                                             t-conorm           \\ \hline
\L ukasiewicz      & $\top_\textup{\L uka}$ &  $\top_\textup{\L uka}(x_1,x_2) = \max(0,x_1+x_2 -1)$      &                                               t-norm             \\ \hline
Bounded sum      & $\bot_\textup{\L uka}$ &  $\bot_\textup{\L uka}(x_1,x_2) = \min(x_1+x_2,1)$     &                                               t-conorm           \\ \hline
Drastic t-norm  & $\top_\textup{D}$& $ \top_\textup{D}(x_1,x_2) = \left \{
\begin{array}{ll}
      x_2 & \textup{if } x_1 = 1  \\
      x_1 & \textup{if } x_2 = 1  \\
      0 & \textup{otherwise} 
\end{array} \right. $     &                                               t-norm             \\ \hline
Drastic t-conorm & $\bot_\textup{D}$ &  $\bot_\textup{D}(x_1,x_2) =  \left \{
\begin{array}{ll}
      x_2 & \textup{if } x_1 = 0  \\
      x_1 & \textup{if } x_2 = 0  \\
      1 & \textup{otherwise} 
\end{array} \right.$      &                                              t-conorm           \\ \hline
\end{tabular}}
\caption{Example of ten aggregation functions. The four last ones are defined from, but their definition can be extended to $[0,1]^n$ for all $n \in \mathbb{N}$ using associativity.}
\label{tab:ex_fonc_aggreg}
\end{table}

\begin{itemize}
\item A t-norm $\top$ is an operator from $[0,1] \times [0,1]$ into $[0,1]$, that satisfies {\bf P2, P4, P6, P9}, with $e_1 = 1$. It follows that {\bf P1, P7, P10}, with $e_0=0$, are also satisfied. The only t-norm that is idempotent ({\bf P5}) is the $\min$. Most usual t-norms are continuous ({\bf P3}), but there are exceptions. 

The last two properties \textbf{P11} and \textbf{P12} require to extend the definition to any number of arguments. For a t-norm, it is straightforward, thanks to associativity (\textbf{P6}), and therefore {\bf P11} and \textbf{P12} are satisfied.

\item A t-conorm $\bot$ is an operator from $[0,1] \times [0,1]$ into $[0,1]$, that satisfies {\bf P2, P4, P6, P9}, with $e_1 = 0$. It follows that {\bf P1, P8, P10}, with $e_0=1$, are also satisfied. The only t-conorm that is idempotent ({\bf P5}) is the $\max$. Most usual t-conorms are continuous ({\bf P3}), but there are exceptions. 

Just like with t-norms, the extension to any number of arguments for a t-conorm is straightforward, thanks to associativity, and {\bf P11} and \textbf{P12} are satisfied.

\item A mean operator $m$ is an operator from $[0,1] \times [0,1]$ into $[0,1]$, that satisfies {\bf P2, P4} and such that $\forall x_1,x_2, \min(x_1,x_2) \leq m(x_1,x_2) \leq \max(x_1, x_2)$ (expressing the compromise), $m \neq \min, m \neq \max$. It follows that it also satisfies {\bf P1, P5}. As the associativity is not necessarily satisfied, the extension of this definition to any number of arguments is not direct. Let us take the example of the arithmetic mean: $\frac{x_1+x_2}{2}$. It can be extended as $\frac{x_1+ ... + x_n}{n}$, which applies when all arguments are available before aggregation. If partial aggregation has to be performed when arguments arrive in a sequence, then another extension could be $\frac{\frac{x_1+x_2}{2} + {x_3}}{2} +\dotsc$ at the price of loosing \textbf{P4}.
\end{itemize}

These three families of operators are respectively modelling a conjunctive, disjunctive and compromise behaviour. Variants of these operators exist, with weaker properties. Examples of functions from those families are given in Table~\ref{tab:ex_fonc_aggreg} and the properties they satisfy are given in Table~\ref{tab:prop_fonc_aggreg}.

Other operators do not fall in one of these three categories. For instance, some symmetrical sums satisfying $1 - \varphi(x_1,x_2) = \varphi(1-x_1, 1-x_2)$ have a hybrid behaviour depending on the values to aggregate. Specifically, they either satisfy the weakening postulate (\textbf{P7}) or the reinforcement postulate (\textbf{P8}), or neither of them, depending on where in the domain the variable to aggregate are located. There are also several non-commutative (\textbf{P4}) aggregation functions in the literature, such as Choquet integrals, with asymmetric capacities~\citep{choquet1954theory}.

\renewcommand{\arraystretch}{1.2}
\begin{table}[t]
\centering
\resizebox{0.9\linewidth}{!}{\begin{tabular}{|c|c|c|c|c|c|c|c|c|c|c|c|c|}
\hline
 & \textbf{P1} & \textbf{P2} & \textbf{P3} & \textbf{P4} & \textbf{P5} & \textbf{P6} & \textbf{P7} & \textbf{P8} & \textbf{P9} & \textbf{P10} & \textbf{P11} & \textbf{P12} \\ \hline
$\textup{avg}_\textup{am}$  & $\checkmark$ & $\checkmark$ & $\checkmark$ & $\checkmark$ & $\checkmark$ & $\times$ & $\times$ & $\times$ & $\times$ & $\times$ & $\checkmark$ & $\checkmark$ \\ \hline
$\textup{avg}_\textup{gm}$  & $\checkmark$ & $\checkmark$ & $\checkmark$ & $\checkmark$ & $\checkmark$ & $\times$ & $\times$ & $\times$  & $\times$ & $e_0 = 0$ & $\checkmark$ & $\checkmark$ \\ \hline
$\top_\textup{prod}$ & $\checkmark$ & $\checkmark$ & $\checkmark$ & $\checkmark$ & $\times$ & $\checkmark$ & $\checkmark$ & $\times$ & $e_1 = 1$ & $e_0 = 0$ & $\checkmark$ & $\checkmark$ \\ \hline
$\bot_\textup{prod}$ & $\checkmark$ & $\checkmark$ & $\checkmark$ & $\checkmark$ & $\times$ & $\checkmark$ & $\times$ & $\checkmark$ & $e_1 = 0$ & $e_0 = 1$ & $\checkmark$ & $\checkmark$ \\ \hline
$\min$ & $\checkmark$ & $\checkmark$ & $\checkmark$ & $\checkmark$ & $\checkmark$ & $\checkmark$ & $\checkmark$ & $\times$ & $e_1 = 1$ & $e_0 = 0$ & $\checkmark$ & $\checkmark$ \\ \hline
$\max$ & $\checkmark$ & $\checkmark$ & $\checkmark$ & $\checkmark$ &$\checkmark$  & $\checkmark$ &  $\times$ & $\checkmark$ &  $e_1 = 0$ & $e_0 = 1$ & $\checkmark$ & $\checkmark$ \\ \hline
$\top_\textup{\L uka}$ & $\checkmark$ & $\checkmark$ & $\checkmark$ & $\checkmark$ & $\times$ & $\checkmark$ & $\checkmark$ &$\times$  & $e_1 = 1$ & $e_0 = 0$ & $\checkmark$ & $\checkmark$ \\ \hline
$\bot_\textup{\L uka}$ & $\checkmark$ & $\checkmark$ & $\checkmark$ & $\checkmark$ & $\times$ & $\checkmark$ & $\times$ & $\checkmark$ &  $e_1 = 0$ & $e_0 = 1$ & $\checkmark$ & $\checkmark$ \\ \hline
$\top_\textup{D}$ & $\checkmark$ & $\checkmark$ & $\times$ & $\checkmark$ & $\times$ & $\checkmark$ & $\checkmark$ & $\times$ & $e_1 = 1$ & $e_0 = 0$ & $\checkmark$ & $\checkmark$ \\ \hline
$\bot_\textup{D}$ & $\checkmark$ & $\checkmark$ & $\times$ & $\checkmark$ & $\times$ & $\checkmark$ &  $\times$& $\checkmark$ &  $e_1 = 0$ & $e_0 = 1$ & $\checkmark$ & $\checkmark$ \\ \hline
\end{tabular}}
\caption{Properties of ten commonly used aggregation functions: $\checkmark$ means the property holds and~$\times$ means it does not.}
\label{tab:prop_fonc_aggreg}
\end{table}

\section{Aggregative Semantics}
\label{sec:agreg}

This section describes our first contribution, the definition of a new gradual semantics family, that we call \emph{aggregative semantics} to emphasise the role of aggregation functions. First, in Section~\ref{subsec:def_aggreg_sem}, we formally define the family of aggregative semantics. Then, in Section~\ref{subsec:comp_mod_aggreg}, we discuss its links and differences with the modular semantics.

\subsection{Definition of the Proposed Aggregative Semantics}
\label{subsec:def_aggreg_sem}

The definition of gradual semantics we propose follows a principle similar to that of modular semantics, with increased modularity: it exploits the bipolarity of the attack and support relations of QBAFs by handling the aggregation of the global weight of the attackers and the global weight of the supporters separately. Indeed, this allows these two components to be considered independently and potentially asymmetrically. This principle of separation had already been proposed by~\citet{cayrol2005gradual} in early work for bipolar argumentation framework, i.e. without weights on arguments. However, the gradual semantics used afterwards dropped this step. This can be seen on the commonly used gradual semantics described in Section~\ref{sec:sem_grad} in which attackers and supporters are simply subtracted (see Equations~\ref{eq:df-quad}, \ref{eq:ebs} and \ref{eq:qe}). 

After this evaluation of the global weight of attackers and supporters, these two values are aggregated with the intrinsic weight of the argument in question. The proposed aggregative semantics thus formally comprises three aggregation functions. 

The first two functions aggregate the multisets of the acceptability degrees of the attackers and supporters, respectively. Since there is generally no specific order on the arguments, these two functions have to be commutative. This hypothesis is discussed in more detail in Section~\ref{sec:discu_agreg}. Furthermore, unlike the classical use of aggregation functions, where the set of values to be aggregated is never empty, in argumentation, the set of attackers or supporters may be empty. This requires the definition of aggregation function to be extended. This extension is discussed in detail in Section~\ref{sec:discu_agreg}. 

Next, in the final aggregation step, an ordered triplet is aggregated, corresponding to the attackers, supporters, and intrinsic weight. Formally the definition is as follows:

\begin{definition}[Aggregative semantics]~
	\label{def:sem_agg}    
	Let $A {= (\mathcal{A},\mathcal{R},\mathcal{S},w)} \in \textup{ac-WAG}$ be a QBAF, and $I_w,I_{\mathcal{R}},I_{\mathcal{S}},I$ pre-ordered sets.  	
	\\ Using the curly bracket notation $\{\}$ to represent  multisets and $\mathcal{M(I)}$ the set of all multisets of~$I$, an aggregative semantics $S$ is a partial function that maps an acyclic QBAF~$A$ to an acceptability function $\accs$ that can be written, for any $a\in \mathcal{A}$, as:
    \begin{align*}
	    \accs(a) =  \varphi_f(&\varphi_{\mathcal{R}}(\{\accs(x) \mid x \in \textup{Att}_a\}), \nonumber\\
        &\varphi_{\mathcal{S}}(\{\accs(y) \mid y \in \textup{Supp}_a\}), w(a))
	\end{align*} 
	
	\noindent where~: \begin{itemize}
		\item $\varphi_{\mathcal{R}}: \displaystyle \mathcal{M}(I) \rightarrow I_{\mathcal{R}}$ is an extended commutative aggregation function computing the global weight of the attackers of an argument;
		\item $\varphi_{\mathcal{S}}: \displaystyle \mathcal{M}(I)\rightarrow I_{\mathcal{S}}$ is an extended commutative aggregation function computing the global weight of the supporters of an argument;
		\item $\varphi_f: I_{\mathcal{R}} \times I_{\mathcal{S}} \times I_w \rightarrow I$ is an aggregation function computing the acceptability degree of an argument from the global weights of its attackers and supporters and its intrinsic strength.
	\end{itemize}

\end{definition}
Denoting the global weights of attackers and supporters as
	\begin{itemize}
		\item $\pi_{\mathcal{R}}(a) =\varphi_{\mathcal{R}}(\{\acc(x) \mid x \in \textup{Att}_a\})$~,
		\item $\pi_{\mathcal{S}}(a) = \varphi_{\mathcal{S}}(\{\acc(y) \mid y \in \textup{Supp}_a\})$,
	\end{itemize}  
the acceptability degree computed by the aggregative semantics is written as:
 \begin{center}
	$\acc(a) =  \varphi_f\left(\pi_{\mathcal{R}}(a), \pi_{\mathcal{S}}(a), w(a)\right).$
\end{center}

Unless specified otherwise, we adopt for the three pre-ordered sets the one most commonly used in the literature, i.e. $I_w = I_{\mathcal{R}} = I_{\mathcal{S}} = I = [0,1]$. 

As in the modular case, the definition of the aggregative semantics is recursive (see Section~\ref{sec:ax_aaf}) and therefore an aggregative semantics is a partial function defined over ac-WAG, the subset of acyclic WAG. We leave the extension to a total function and the convergence study in the case of cyclic graphs for future work. Here we focus the discussion on the impact and relevance of the properties of the aggregation functions.

\begin{example}\label{ex:agg_sem}
    Let us consider the toy example represented in Figure~\ref{fig:QBAF} and an aggregative semantics based on the following three aggregation functions: let us take $\top_{\textup{prod}}$ for $\varphi_{\mathcal{R}}$ and $\bot_D$ for $\varphi_{\mathcal{S}}$ extended to the empty set so that the stability principle (\textbf{A5}) holds, i.e.${\acc(a) = w(a)}$ if $\textup{Att}_a = \textup{Supp}_a = \emptyset$.
        
    For $\varphi_f$, let us take an adapted version of an arithmetic mean extended to three arguments: $\varphi_f(x,y,z) = \frac{\frac{1-x + y}{2}+z}{2}$. It takes $1-x$ in the computation as the attackers should decrease the acceptability of an argument. We discuss this later in the paper in Section~\ref{subsec:aggreg_final}.

    Since the graph is acyclic, we first consider the unattacked and unsupported arguments, $b,c,d,e$ here. For each of them, the stability principle states that their acceptability degree equals their intrinsic weights, e.g. $\acc(b) = w(b) = 0.9$.
    
    As a consequence, when propagating to $a$, we get: \begin{itemize}
        \item ${\pi_{\mathcal{R}}(a) = \varphi_{\mathcal{R}}(0.9,0.1) =\top_{\textup{prod}}(0.9,0.1) = 0.09}$
        \item $\pi_{\mathcal{S}}(a) = \varphi_{\mathcal{S}}(0.2,0.8) = \bot_D(0.2,0.8) = 1$
        \item $\acc(a) = \frac{\frac{1-0.09+1}{2} + 0.5}{2} = 0.73$
    \end{itemize}
    
    In contrast to DF-Quad, Ebs, and QE for which $\acc(a)$ is around $0.5$ (see Example~\ref{ex:grad_sem} in Section~\ref{sec:sem_grad}), this aggregative semantics favours the supporters, as shown by the acceptability degree of $a$, which is much higher than $0.5$. 
\end{example}

An advantage of using aggregation functions to compute the acceptability degree of an argument is that these functions have been applied in many areas of research with a variety of input and output formats. Consequently, this definition is compatible with many different types of variables to aggregate, corresponding to a great diversity in the pre-ordered sets used. As a result, this family of gradual semantics can be used for enriched versions of QBAF. For example, if one wants to add a timestamp, for instance in $\mathbb{R}^+$, to the arguments, then $I$ can be set to~${[0,1] \times \mathbb{R}^+}$.

In order to be able to build an appropriate aggregative semantics in a given context, we propose in Sections~\ref{subsec:phi_etoile} and \ref{subsec:aggreg_final} a discussion to guide the choice of the three aggregation functions based on the properties they verify.

\subsection{Links and Differences with Modular Semantics}
\label{subsec:comp_mod_aggreg}

Breaking down the calculation of the acceptability degree of an argument into several ``elementary'' steps allows us to explicitly disentangle attackers and supporters, offering multiple advantages: keeping the bipolarity one step further and creating a more understandable gradual semantics. Therefore, it extends the idea of modular semantics~(see Section~\ref{sec:modular_sem}).

Since modular semantics impose no restriction on $\alpha$ and $i$, it can lead to any function, if it is built in an ad hoc way. This makes a formal comparison between the modular and aggregative semantics irrelevant. However, this ad hoc definition leads to a loss of intelligibility compared to the explicit disentanglement offered by an aggregative semantics. 

If $\alpha$ is restricted to a set of aggregation functions with intelligible behaviour for computing a global weight for the parents of an argument, then the proposed aggregative approach is more general. Indeed, the aggregative semantics decomposes the $\alpha$ function into two components that can be designed independently, preserving the specificities of attack and support relations, and choosing their respective properties according to the desired behaviour. Conversely, the most commonly used semantics can be reformulated in terms of an aggregative semantics. It can be proved directly by syntactic manipulation of the formulas, as stated in the following proposition. 
 
\begin{proposition}\label{prop:mod_to_agg}
    DF-Quad~\citep{rago2016discontinuity}, Euler-based~\citep{amgoud2018weighted} and Quadratic Energy~\citep{potyka2018continuous} semantics are aggregative semantics.

\end{proposition}

\begin{proof}The proof is done directly by defining the three aggregation functions corresponding to the associated gradual semantics. 

Given a multiset $M$ of values in $I= [0,1]$ and a semantics $S \in \{\textup{DF-Quad, Ebs, QE}\}$, the induced acceptability function $\accs$ can respectively be written using the following aggregation functions: \begin{itemize}
    \item DF-Quad (see Equation~\ref{eq:df-quad}):\begin{itemize}
        \item $\varphi_{\mathcal{R}}(M) = \varphi_{\mathcal{S}}(M) = \displaystyle \prod_{m \in M} (1-m)$, with $I_\mathcal{R} = I_\mathcal{S} = [0,1]$
        \item $\varphi_f^{\textup{DF-Quad}}(x,y,z) = z - z\cdot\max(0,y-x)+ (1-z)\cdot\max(0,-(y-x))$
    \end{itemize}
    \item Ebs (see Equation~\ref{eq:ebs}):\begin{itemize}
        \item $\varphi_{\mathcal{R}}(M) = \varphi_{\mathcal{S}}(M) = \displaystyle \sum_{m \in M} m$, with $I_\mathcal{R} = I_\mathcal{S} = \mathbb{R}^+$
        \item $\varphi_f^{\textup{Ebs}}(x,y,z) = \displaystyle 1 - \frac{1-z^2}{1 + z\cdot e^{y-x}}$
    \end{itemize}
    \item QE (see Equation~\ref{eq:qe}):\begin{itemize}
        \item $\varphi_{\mathcal{R}}(M) = \varphi_{\mathcal{S}}(M) = \displaystyle \sum_{m \in M} m$, with $I_\mathcal{R} = I_\mathcal{S} = \mathbb{R}^+$
        \item $\varphi_f^{\textup{QE}}(x,y,z) = z \cdot\left( 1 - \frac{\max(0,-(y-x))^2}{1+\max(0,-(y-x))^2}\right) + (1-z)\cdot \frac{\max(0,y-x)^2}{1+\max(0,y-x)^2}$
    \end{itemize}
\end{itemize}
\end{proof}

A first observation induced by this proposition is that, in all three cases, attackers and supporters are aggregated together through the difference between the global weight of the supporters and that of the attackers. Moreover, these two values are obtained from the same function, i.e. $\varphi_{\mathcal{R}} = \varphi_{\mathcal{S}}$. Therefore, attackers and supporters are treated symmetrically in these three common gradual semantics of the literature.

We argue that imposing $\varphi_{\mathcal{R}} = \varphi_{\mathcal{S}}$ is a restrictive condition and that in some scenarios we need to be able to treat attackers and supporters differently. For example, to model a judge's reasoning during a trial, an asymmetry between an argument supporting a person's innocence and another supporting their guilt, i.e. attacking their innocence, is required. Under the principle of the presumption of innocence, the prosecution must provide evidence of the suspect's guilt beyond a reasonable doubt. Using an aggregative semantics, this asymmetry can be taken into account at two levels. First, when computing the global strength of attackers and supporters, stronger or a greater number of attackers can be required to get the same or a higher value than the global strength of the supporters. Secondly, when calculating the acceptability of the argument, the role of attackers and supporters may be different. Thus, in contrast to the commonly used gradual semantics, we maintain bipolarity in the final stage of the calculation, with one value for the attackers and one for the supporters.

Aggregative semantics open the way to new gradual semantics for QBAFs, depending on the aggregation functions chosen to instantiate Definition~\ref{def:sem_agg}. The next section examines the impact of classical properties of these functions \mj{on the induced semantics}. 
Section~\ref{sec:principles}, conversely, examines the correspondence of the classical principles of argumentative semantics with these properties. 

\section{A Discussion on the Desirable Properties of the Aggregation Functions}
\label{sec:discu_agreg}

In this section we examine some properties of aggregation functions in the context of argumentation, in order to guide the instantiation of Definition~\ref{def:sem_agg}, i.e. to allow for the choice of the most suitable aggregation functions among the multitude of existing ones: we discuss the relevance of each of these properties in the specific context of argumentation; in particular, we identify and illustrate whether these properties are required or desirable only in certain cases.

We first discuss functions $\varphi_{\mathcal{R}}$ and $\varphi_{\mathcal{S}}$ before turning to the case of $\varphi_{f}$.

\subsection{The Case of $\varphi_{\mathcal{R}}$ and $\varphi_{\mathcal{S}}$}
\label{subsec:phi_etoile}

This section focuses on the desirable properties for the computation of the global weight of the attackers and supporters of an argument, i.e.~$\varphi_{\mathcal{R}}$ and $\varphi_{\mathcal{S}}$. In the following, these two functions are grouped under the notation $\varphi_*$ with $* \in \{\mathcal{R},\mathcal{S}\}$. Moreover, unless the distinction between attackers and supporters is necessary, we use the case of attackers to explain and discuss the postulates. We consider $A = (\mathcal{A}, \mathcal{R}, \mathcal{S}, w)$ an acyclic QBAF, and $a \in \mathcal{A}$, one of its arguments.

{\textbf{(P1): Boundary condition of $\varphi_{*}$ --}} This postulate deals with the expected behaviour at the limits of the definition domain, i.e. the values of $\varphi_{*}$ for ${\vec{0} = (0,\dotsc,0)}$ and $\vec{1} = (1,\dotsc,1)$, when $I = [0,1]$. In argumentation, if an argument is only attacked by arguments with acceptability $0$, there is no a priori reason for the global weight of the attackers to take a value other than zero. Thus, we expect $\varphi_*(\vec{0}) = 0$. Similarly, if an argument is only attacked by arguments with the highest acceptability degree, i.e. $1$, then we expect the highest global weight for the attackers, formally $\varphi_*(\vec{1}) = 1$. 

In argumentation, there is another boundary case to consider, namely the case of an empty set of values to aggregate. It corresponds to arguments that have either no attackers or no supporters. However, by definition, an aggregation function combines a non-empty multiset of values into a single value. For this reason, Definition~\ref{def:sem_agg} uses the notion of extended aggregation functions, dealing with the case of an empty set as well. The choice of the value taken by~$\varphi_*(\emptyset)$ cannot be treated independently of the choice of $\varphi_f$. Indeed, if an argument has no attackers, adding an attacker can, at best, have no effect on the acceptability degree of the argument or must decrease it. A similar reasoning applies when an argument has no supporters. Formally, $\forall x,y,z \in [0,1]^3$: \[\varphi_f(\varphi_{\mathcal{R}}(\emptyset),y,z) \geq \varphi_f(x,y,z) \geq \varphi_f(x,\varphi_{\mathcal{S}}(\emptyset),z)\]

As explained later in Section~\ref{subsec:aggreg_final}, $\varphi_f$ is non-increasing according to its first variable and non-decreasing according to its second variable. Therefore, it makes sense to take, if it exists, the minimum of the codomain of $\varphi_{\mathcal{R}}$ as the value for $\varphi_{\mathcal{R}}(\emptyset)$. Similarly, for the value of $\varphi_{\mathcal{S}}(\emptyset)$, if it exists, the minimum of the codomain of $\varphi_{\mathcal{S}}$ can be considered.

For all the functions listed in Table~\ref{tab:ex_fonc_aggreg}, these values are $0$.
  
\textbf{(P2): Non-decreasingness of $\varphi_{*}$ --} In argumentative terms, this postulate states that if the acceptability of one of the attackers of $a$ increases, then the total weight of its attackers increases or remains constant. This is the behaviour generally expected in argumentation~\citep{cayrol2005gradual}. 

\textbf{(P3): Continuity of $\varphi_{*}$ --} Imposing that $\varphi_{\mathcal{R}}$ satisfies the continuity postulate means that a small change in the acceptability of one attacker can induce only a small change in the global weight of the attackers. This behaviour can indeed be considered desirable.

However, in certain scenarios, threshold effects may be desirable: above a certain global weight for the attackers, for example, there may be saturation,  leading to a sudden jump corresponding to a discontinuity. Assuming a maximum countable number of such thresholds, a desirable postulate is that the $\varphi_{*}$ functions are piecewise continuous. We give an example of such a behaviour in Section~\ref{subsec:aggreg_final}. 

\textbf{(P4): Commutativity of $\varphi_{*}$ --} 
In abstract argumentation and for a classical QBAF, this is an expected property for gradual semantics~\citep{cayrol2005gradual}.

However, if for instance temporality has to be taken into account, attackers do not play symmetric roles. Indeed, in the construction of a debate for instance, the order in which arguments are enunciated plays a crucial role: 
an argument enunciated at the very beginning or the very end of a debate most often has more impact than the ones enunciated in the middle~\citep{miller1959recency}. In such cases, the arguments cannot be swapped, the commutativity property is then undesirable. 
An order on the arguments may also be induced by moral values, as in value-based argumentation frameworks~\citep{bench2002value}.

Thus, Definition~\ref{def:sem_agg} needs to be relaxed, when an order between arguments should be taken into account.

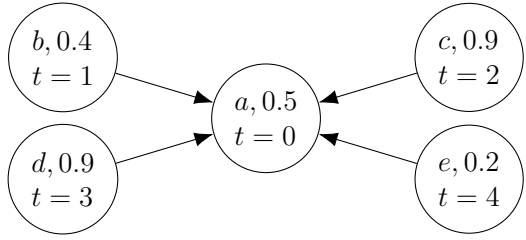
\begin{figure}[t]
	\centering
	\begin{tikzpicture}[scale=0.9,transform shape, node distance={25mm}, main/.style = {draw, circle,minimum size=12mm}]
		\node[main, align=left] (1) {$a,0.5$ \\ $t = 0$};
		
		\node[main,align=left] (2)  at (-3,0.9) {$b,0.4$ \\ $t=1$};
		\node[main,align=left] (3)  at (-3,-0.9) {$d,0.9$ \\ $t=3$};
		\node[main,align=left] (4)  at (3,0.9) {$c,0.9$ \\ $t=2$};
		\node[main,align=left] (5)  at (3,-0.9) {$e,0.2$ \\ $t=4$};
		
		\draw [arrows = {-Latex[scale=1.5]}] (2) -- (1);
		\draw [arrows = {-Latex[scale=1.5]}] (3) -- (1);
		\draw [arrows = {-Latex[scale=1.5]}] (4) -- (1);
		\draw [arrows = {-Latex[scale=1.5]}] (5) -- (1);

	\end{tikzpicture}
	\caption{Illustrative example for the commutativity postulate (\textbf{P4}).}
	\label{fig:commutativity}
\end{figure}

\begin{example}
    To illustrate the impact commutativity can have when computing the weight of the attackers of an argument, we consider the toy example represented in Figure~\ref{fig:commutativity}. In this QBAF, four attackers are enunciated at different dates~$t$ between~$1$ and~$4$, here in the order b, c, d, e. We consider two aggregation function for $\varphi_\mathcal{R}$: a commutative average and a non-commutative weighted average.
    
    For the commutative case, $\varphi_{\mathcal{R}}^c(\mathbf{x}_n) = \textup{avg}_\textup{am}(\mathbf{x}_n)$ leading to $\varphi_{\mathcal{R}}(a) = 0.6$.
    
    For the non-commutative case, we consider $\varphi_{\mathcal{R}}^{nc}(\mathbf{x_n}) = \frac{1}{\sum \alpha_i} \sum \alpha_i x_i$ with $ \alpha_i = (1/10)^i$  if $i \leq n/2$ and $\alpha_i = \alpha_{n-i}$ otherwise. It gives more importance to arguments at the beginning or at the end.  It leads to $\varphi_{\mathcal{R}}^{nc}(a) = 0.35$. Now, if instead of b, c, d, e, we consider a different order, for instance c, b, e, d, we get that~$\varphi_{\mathcal{R}}^{nc}(a) = 0.85$.

\end{example}

\textbf{(P6): Associativity of $\varphi_{*}$ --} This is mainly a structural characteristic, that makes it possible to aggregate existing data and then update the result when new data  arrive. In argumentation, when the argument temporal order is not taken into account, all the information is already present in the graph when aggregation takes place. In real-time application scenarios where the speed of computation matters, using the associativity of $\varphi_{*}$ can save time. However, this remains negligible in the case of acyclic graphs, where convergence is guaranteed in a maximum of $N_{arg}$ operations, where $N_{arg}$ is the total number of arguments present in the graph~\citep{potyka2019extending}. Since imposing associativity restricts the possible operators~\citep{grabisch2009aggregation} for almost no impact on the computation of $\acc$, we do not impose associativity in general.

\textbf{(P5): Idempotence of $\varphi_{*}$ --}  The idempotence property is generally desirable when the sources of the values to be aggregated are not independent: in this case, if the information is redundant and all attackers have the same value, then the aggregated value should be the same. By contrast, if the sources are independent, then redundant information can reinforce or weaken each other. We discuss this behaviour together with postulate \textbf{P6}.

In argumentation, the notion of information redundancy depends on the nature of the arguments and the modelled context. An example of such a situation is a debate during which an argument is repeated by several independent speakers.

Furthermore, if $\varphi_{*}$ is associative and idempotent, the repetition of an argument has no effect:  $\varphi_{*}(x,x,y) = \varphi_{*}(\varphi_{*}(x,x),y) = \varphi_{*}(x,y)$. More generally, for such a $\varphi_*$, giving arguments with identical degrees of acceptability is useless. Therefore, the negation of at least the idempotence postulate (\textbf{P5}) or the associativity postulate (\textbf{P6}) is necessary.

\begin{example}
    To illustrate this principle, we take the  simple example represented in Figure~\ref{fig:idem}. We consider two functions: first $\varphi_\mathcal{R}=\max$, that is idempotent, leads to $\pi_{\mathcal{R}}(a) = 0.4$. Secondly, $\varphi_\mathcal{R}= \bot_{\textup{\L uka}}$, that is not idempotent, leads to $\pi_\mathcal{R}(a) = 0.8$. This difference reflects the fact that the $\max$ function does not account for the number of repeated arguments or the number of arguments in general. Conversely, the bounded sum $\bot_{\textup{\L uka}}$ does take this into account until the aggregation reaches $1$.
\end{example}

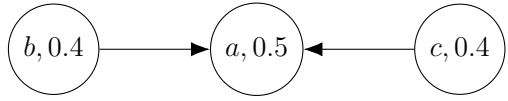
\begin{figure}[t]
	\centering
	\begin{tikzpicture}[scale=0.9,transform shape, node distance={25mm}, main/.style = {draw, circle,minimum size=12mm}]
		\node[main, align=left] (1) {$a,0.5$};
		
		\node[main,align=left] (2)  at (-3,0) {$b,0.4$};
		\node[main,align=left] (3)  at (3,0) {$c,0.4$};
		
		\draw [arrows = {-Latex[scale=1.5]}] (2) -- (1);
		\draw [arrows = {-Latex[scale=1.5]}] (3) -- (1);

	\end{tikzpicture}
	\caption{Illustrative example for the idempotence postulate (\textbf{P5}).}
	\label{fig:idem}
\end{figure}

\textbf{(P7-8): Weakening and Reinforcement of $\varphi_{*}$ --} The weakening (\textbf{P7}) and reinforcement (\textbf{P8}) postulates determine the global behaviour of the function, i.e. if the function is conjunctive, disjunctive or a compromise. In particular, those properties characterise the behaviour of an aggregation function in the presence of \textit{weak} and \textit{strong} values, where the meaning of strong and weak varies according to the context. When applied to argumentation, the weakening postulate for instance reflects the influence that arguments with a low degree of acceptability have over other attackers or supporters of the same argument. Formally, the global weight of the attackers is less than or equal to the degree of acceptability of the weakest attacker.

A related property exists in argumentation in the case of arguments having a degree of acceptability close to the limits of the domain of definition: the monotony principle~(\textbf{A7}). In contrast to the monotony postulate~(\textbf{P2}) that considers a constant number of attackers and supporters, \textbf{A7} expresses the fact that an argument that is less attacked or more supported has a stronger acceptability, all other things being equal. It implies that adding an attacker, regardless of its degree of acceptability, can only increase or leave unchanged the global weight of the attackers. We believe that this behaviour is not always desirable, and depends in particular on the additional attacker: a ``stupid'' or unacceptable argument could backfire on its own side by reducing the credibility of the attacker's side. Therefore, another possible behaviour is that the addition of such an argument may weaken the set of attackers, i.e. reduce the value of the global weight of the set of arguments. In this case, the weakening postulate is desired.
Similarly, the reinforcement postulate states that ``brilliant'' arguments reinforce each other, increasing their global weight.

\begin{example} We consider the QBAF represented in Figure~\ref{fig:weak_reinf} and three aggregation functions for $\varphi_\mathcal{R}$: a t-norm $\varphi_{\mathcal{R}}^1 = \min$, a t-conorm $\varphi_{\mathcal{R}}^2 = \max$, and an average operator $\varphi_{\mathcal{R}}^3 = \textup{avg}_{\textup{am}}$. Applying these three functions to compute the global weight of the attackers of $a$ leads to: $\pi_{\mathcal{R}}^1(a) = 0.1$, $\pi_{\mathcal{R}}^2(a) = 0.9$ and $\pi_{\mathcal{R}}^3(a) = 0.5$, ranging over the whole interval.

    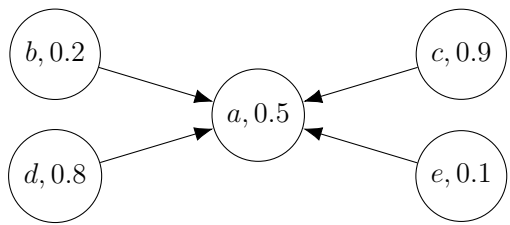
\begin{figure}[t]
	\centering
	\begin{tikzpicture}[scale=0.9,transform shape, node distance={25mm}, main/.style = {draw, circle,minimum size=12mm}]
		\node[main, align=left] (1) {$a,0.5$};
		
		\node[main,align=left] (2)  at (-3,0.9) {$b,0.2$};
		\node[main,align=left] (3)  at (-3,-0.9) {$d,0.8$};
		\node[main,align=left] (4)  at (3,0.9) {$c,0.9$};
		\node[main,align=left] (5)  at (3,-0.9) {$e,0.1$};
		
		\draw [arrows = {-Latex[scale=1.5]}] (2) -- (1);
		\draw [arrows = {-Latex[scale=1.5]}] (3) -- (1);
		\draw [arrows = {-Latex[scale=1.5]}] (4) -- (1);
		\draw [arrows = {-Latex[scale=1.5]}] (5) -- (1);

	\end{tikzpicture}
	\caption{Illustrative example for the weakening and reinforcement postulates (\textbf{P7-8}).}
	\label{fig:weak_reinf}
    \end{figure}

In certain contexts, a hybrid behaviour can be more desirable. Let us define the following associative symmetrical sum: \[ \varphi_\mathcal{R}^4(x_1,x_2) = \displaystyle \frac{x_1x_2}{1 - x_1-x_2+2x_1x_2}\] With this function, the weakening postulate holds if $x_1 < 0.5$ and $x_2 <0.5$ whereas the reinforcement postulate holds if $x_1>0.5$ and $x_2 > 0.5$. In the middle, neither of the two holds and the function behaves as a compromise operator: $\pi_\mathcal{R}^4(a) = \varphi_\mathcal{R}^4(0.2,0.9,0.8,0.1) = \varphi_\mathcal{R}^4(\varphi_\mathcal{R}^4(0.2,0.9),\varphi_\mathcal{R}^4(0.8,0.1)) = 0.5.$

\end{example}

\textbf{(P9): Neutral element of $\varphi_*$ --} By definition, a neutral element is a variable that has no effect on the final result of the aggregation. In argumentation this corresponds to situations where enunciating or not enunciating an argument makes no difference. This property lends a special meaning to one specific degree of acceptability, implying the definition of a special value in the domain $I$ of the acceptability degrees.

A first solution is to create an out-of-scale value, $e$, corresponding to this special meaning. It requires to extend the definition of $\varphi_f$ so that this value $e$ can be reached and used.

Another option is to select a value from $I$, the domain of $\acc$, as the neutral element. For example, if the  neutrality principle~(\textbf{A6}) is verified, the behaviour of the semantics with respect to ``useless'' arguments is ignorance, i.e.  an attacker with an acceptability degree zero has no effect, as if it were not present. In our formalism, this is equivalent to choosing an operator $\varphi_*$ for which $0$ is a neutral element~(see Proposition~\ref{prop:qbaf_neutre}), excluding the choice of t-norms for $\varphi_*$. However, this choice implies that there is no difference between unacceptable arguments and useless ones, as both have an acceptability degree equal to $0$. On another note, if, in addition, $\varphi_*$ is continuous at $0$, it leads to either the weakening (\textbf{P7}) postulate not holding or $\varphi_*$ being the constant function $0$. If, instead, $\varphi_*$ is not continuous at $0$, then it brings us back to the first solution where $e=0$. 

To conclude about this postulate, the existence of a neutral element is context dependant. 

\textbf{(P10): Null element of $\varphi_*$ --} A null element is a variable that determines the outcome of the aggregation, regardless of the values of the other variables. In argumentation, it corresponds to an argument whose acceptability prevails over those of the other arguments. It can, for example, represent irrefutable arguments.

As with the previous property, there are two alternatives: either adding an out-of-scale value or selecting a specific one from the domain. A natural candidate would be to take the value $1$ and consequently mix up the notion of irrefutability and total acceptability. In contrast to the discussion on the neutral element, the difference between these two types of argument lies in their nature rather than their expected effect. An irrefutable argument is one that cannot be attacked and whose acceptability is always at the maximum level. Conversely, a standard argument may have an acceptability value of $1$ at one point in the debate, but then become less acceptable as the discussion progresses. In both cases, those two arguments can have an acceptability degree that serve as null element.

As a conclusion, the necessity of this postulate depends on the context.

\textbf{(P11): Composition of $\varphi_*$ --} Consider the addition of an argument $c$ that attacks both $a$ and $b$. None of the previous postulates tells us how the global weight of the attackers of $a$ will change in comparison to that of $b$. If the composition postulate holds, and if $\pi_\mathcal{R}(a)\geq\pi_\mathcal{R}(b)$, then the addition of argument $c$ will not change the relative order between $\pi_\mathcal{R}(a)$ and $\pi_\mathcal{R}(b)$. We argue that this is a desirable behaviour. Indeed, the relative order between $\pi_\mathcal{R}(a)$ and $\pi_\mathcal{R}(b)$ being reversed after adding a common attacker is not intuitive for a user. It could lead to a loss of understandability for the aggregative semantics. Imposing the composition postulate prevents this scenario from happening. 

\begin{example} To illustrate what could happen if the composition postulate does not hold, we consider the QBAF represented on the left side of Figure~\ref{fig:composition} and an aggregation function $\varphi_\mathcal{R}$ that does not verify the composition principle, defined as:
\[   \varphi_\mathcal{R}(x_1,\dotsc, x_n) = \left\{
\begin{array}{ll}
      \text{avg}_\textup{am}{(x_1,\dotsc,x_n)} & \textup{if } \text{ avg}_\textup{am}{(x_1,\dotsc,x_n)} > 0.5  \\
      \min(x_1,\dotsc,x_n) & \textup{otherwise.} 
\end{array} 
\right. \]

\noindent One then gets $\pi_{\mathcal{R}}(a_1) = 0.2$ and $\pi_\mathcal{R}(a_2) = 0.3$, i.e. $a_2$ is more attacked than $a_1$.

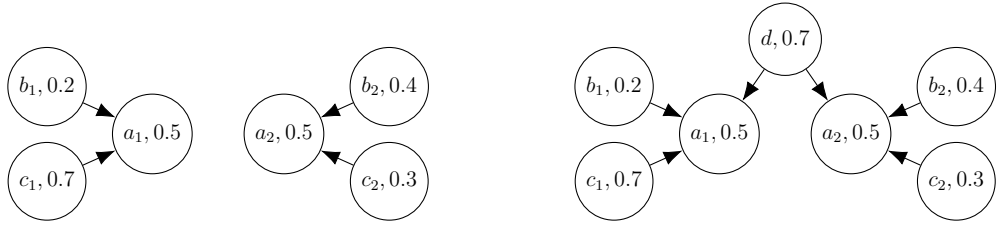
\begin{figure}[t]
	\centering

    \begin{subfigure}{0.45\textwidth}
    \centering
        \begin{tikzpicture}[scale=0.7,transform shape, node distance={25mm}, main/.style = {draw, circle,minimum size=12mm}]
		\node[main, align=left] (1) {$a_1,0.5$};
        \node[main, right of = 1] (10) {$a_2,0.5$};
		
		\node[main,align=left] (2)  at (-2,0.9) {$b_1,0.2$};
		\node[main,align=left] (3)  at (-2,-0.9) {$c_1,0.7$};
		\node[main,align=left] (4)  at (4.5,0.9) {$b_2,0.4$};
		\node[main,align=left] (5)  at (4.5,-0.9) {$c_2,0.3$};
		
		\draw [arrows = {-Latex[scale=1.5]}] (2) -- (1);
		\draw [arrows = {-Latex[scale=1.5]}] (3) -- (1);
		\draw [arrows = {-Latex[scale=1.5]}] (4) -- (10);
		\draw [arrows = {-Latex[scale=1.5]}] (5) -- (10);
				
	\end{tikzpicture}
    \end{subfigure}%
    \hfill
    \begin{subfigure}{0.45\textwidth}
    \centering
    \begin{tikzpicture}[scale=0.7,transform shape, node distance={25mm}, main/.style = {draw, circle,minimum size=12mm}]
		\node[main, align=left] (1) {$a_1,0.5$};
        \node[main, right of = 1] (10) {$a_2,0.5$};
		
		\node[main,align=left] (2)  at (-2,0.9) {$b_1,0.2$};
		\node[main,align=left] (3)  at (-2,-0.9) {$c_1,0.7$};
		\node[main,align=left] (4)  at (4.5,0.9) {$b_2,0.4$};
		\node[main,align=left] (5)  at (4.5,-0.9) {$c_2,0.3$};
		\node[main,align=left] (6)  at (1.25,1.8) {$d,0.7$};
		
		\draw [arrows = {-Latex[scale=1.5]}] (2) -- (1);
		\draw [arrows = {-Latex[scale=1.5]}] (3) -- (1);
		\draw [arrows = {-Latex[scale=1.5]}] (4) -- (10);
		\draw [arrows = {-Latex[scale=1.5]}] (5) -- (10);
		\draw [arrows = {-Latex[scale=1.5]}] (6) -- (1);
		\draw [arrows = {-Latex[scale=1.5]}] (6) -- (10);

	\end{tikzpicture}
    \end{subfigure}%

	\caption{Illustrative example for the composition postulate (\textbf{P11}).}
	\label{fig:composition}
\end{figure}

Now, a new argument $d$ attacking both $a_1$ and $a_2$ is enunciated, as represented on the right side of Figure~\ref{fig:composition}. Then, we get that $\pi_{\mathcal{R}}(a_1) = 0.54$ and $\pi_\mathcal{R}(a_2) = 0.3$. Thus, $a_1$ has become more attacked than $a_2$ even if the same attacker has been added, which is not intuitive for a user.
\end{example}

\textbf{(P12): Decomposition of $\varphi_*$ --} Reciprocally, the decomposition postulate implies that if an attacker of two arguments is withdrawn, then the order between the global weights of the attackers of each of these two arguments is preserved. In the case where temporality is taken into account, such a scenario occurs when a forgetting mechanism by argument withdrawal is desired. An argument that has been enunciated ``a long time ago'' may be forgotten after a certain period if it has not been enunciated again, i.e. removed from the values to be aggregated. 

Another context in which this situation can occur is when $0$ is a neutral element of $\varphi_*$ functions. Making an argument unacceptable, i.e. setting its acceptability to zero, then amounts to removing the argument from the aggregation: ${\varphi_*(\vec{x_n},0) = \varphi_{*}(\vec{x_n})}$. In these two situations, the order between the global strengths of each of the sets is expected to be preserved, which is equivalent to the decomposition postulate holding. 

Consequently, the decomposition principle should hold in contexts where arguments can be removed; otherwise, this property does not matter.

\begin{example}
To illustrate what could happen if the composition postulate does not hold, we consider the QBAF represented on the left side of Figure~\ref{fig:composition} and an aggregation function $\varphi_\mathcal{R}$ that does not verify the decomposition principle, defined as:

\[   \varphi_\mathcal{R}(x_1,\dotsc, x_n) = \left\{
\begin{array}{ll}
      \max(x_1,\dotsc,x_n) & \textup{if } \textup{ avg}_\textup{am}{(x_1,\dotsc,x_n)} > 0.5  \\
      \text{avg}_\textup{am}{(x_1,\dotsc,x_n)} & \textup{otherwise.} 
\end{array} 
\right. \]

\noindent One gets $\pi_{\mathcal{R}}(a_1) = 0.6$ and $\pi_\mathcal{R}(a_2) = 0.8$. So, $a_2$ is more attacked than~$a_1$.

Now, let us consider that argument $d$ is removed, for instance $d$ is forgotten, as represented on the right side of Figure~\ref{fig:decomposition}. Then, we get $\pi_{\mathcal{R}}(a_1) = 0.6$ and $\pi_\mathcal{R}(a_2) = 0.5$. Thus, $a_1$ has become more attacked than $a_2$ even if it is a common attacker that has been removed.
\end{example}

\begin{figure}[t]
	\centering

    \begin{subfigure}{0.45\textwidth}
    \centering
        \begin{tikzpicture}[scale=0.7,transform shape, node distance={25mm}, main/.style = {draw, circle,minimum size=12mm}]
		\node[main, align=left] (1) {$a_1,0.5$};
        \node[main, right of = 1] (10) {$a_2,0.5$};
		
		\node[main,align=left] (2)  at (-2,0.9) {$b_1,0.6$};
		\node[main,align=left] (3)  at (-2,-0.9) {$c_1,0.6$};
		\node[main,align=left] (4)  at (4.5,0.9) {$b_2,0.8$};
		\node[main,align=left] (5)  at (4.5,-0.9) {$c_2,0.2$};
		\node[main,align=left] (6)  at (1.25,1.8) {$d,0.6$};
		
		\draw [arrows = {-Latex[scale=1.5]}] (2) -- (1);
		\draw [arrows = {-Latex[scale=1.5]}] (3) -- (1);
		\draw [arrows = {-Latex[scale=1.5]}] (4) -- (10);
		\draw [arrows = {-Latex[scale=1.5]}] (5) -- (10);
		\draw [arrows = {-Latex[scale=1.5]}] (6) -- (1);
		\draw [arrows = {-Latex[scale=1.5]}] (6) -- (10);
				
	\end{tikzpicture}
    \end{subfigure}%
    \hfill
    \begin{subfigure}{0.45\textwidth}
    \centering
    \begin{tikzpicture}[scale=0.7,transform shape, node distance={25mm}, main/.style = {draw, circle,minimum size=12mm}]
		\node[main, align=left] (1) {$a_1,0.5$};
        \node[main, right of = 1] (10) {$a_2,0.5$};
		
		\node[main,align=left] (2)  at (-2,0.9) {$b_1,0.6$};
		\node[main,align=left] (3)  at (-2,-0.9) {$c_1,0.6$};
		\node[main,align=left] (4)  at (4.5,0.9) {$b_2,0.8$};
		\node[main,align=left] (5)  at (4.5,-0.9) {$c_2,0.2$};
		
		\draw [arrows = {-Latex[scale=1.5]}] (2) -- (1);
		\draw [arrows = {-Latex[scale=1.5]}] (3) -- (1);
		\draw [arrows = {-Latex[scale=1.5]}] (4) -- (10);
		\draw [arrows = {-Latex[scale=1.5]}] (5) -- (10);

	\end{tikzpicture}
    \end{subfigure}%

	\caption{Illustrative example for the composition principle (\textbf{P12}).}
	\label{fig:decomposition}
\end{figure}
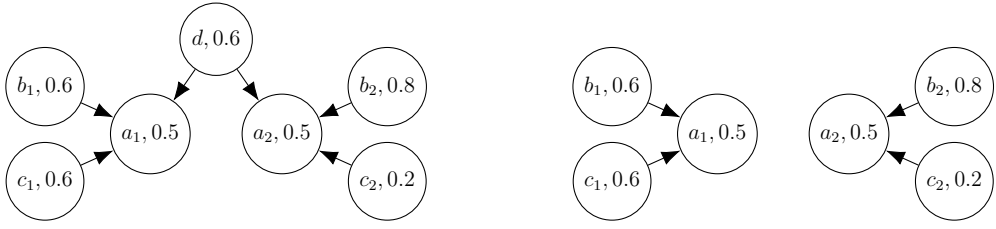
\paragraph{Partial conclusion} The boundary conditions (\textbf{P1}), monotony (\textbf{P2}), composition~(\textbf{P11}) and decomposition (\textbf{P12}) are postulates that we consider desirable regardless of the context. Associativity (\textbf{P6}) is not a property that is necessarily expected and it is never desirable in cases where the postulate of idempotence is verified. The postulates of continuity (\textbf{P3}), commutativity (\textbf{P4}), idempotence (\textbf{P5}), weakening (\textbf{P7}), reinforcement (\textbf{P8}), neutral element (\textbf{P9}) and null element (\textbf{P10}) depend on the cases in which the aggregative semantics is applied. Finally, if the neutral element (\textbf{P9}) and null element (\textbf{P10}) postulates hold, then their application depends on the choice of $\varphi_f$. 

\subsection{The Case of $\varphi_f$}\label{subsec:aggreg_final}

This subsection discusses desirable postulates for $\varphi_f$, which computes the degree of acceptability of an argument as a function of the global weight of its attackers, the global weight of its supporters, and its intrinsic weight.
Unlike~$\varphi_{\mathcal{R}}$ and~$\varphi_{\mathcal{S}}$  that process a varying number of variables, $\varphi_f$ has exactly three variables. Hence, the postulates 
related to adding or removing a variable, i.e. commutativity (\textbf{P4}), associativity (\textbf{P6}), the existence of a neutral element (\textbf{P9}), composition (\textbf{P11}) and decomposition~(\textbf{P12}), are irrelevant.

\textbf{(P2): Monotony of $\varphi_f$ --} Each of the variables in $\varphi_f$ has a specific meaning. The first one corresponds to the weight of the attackers $\pi_\mathcal{R}$ of an argument: the greater it is, the lower the acceptability of the argument, which requires a non-increasing function. The second and third variables correspond to the weight of the supporters $\pi_\mathcal{R}$ and the intrinsic strength of the argument $w$. The greater they respectively are, the higher the acceptability, calling for non-decreasing functions. Thus, $\varphi_f$ is expected to be non-increasing according to its first variable and non-decreasing according to the last two.

\textbf{(P1): Boundary condition of $\varphi_f$ --} Due to the monotony of $\varphi_f$, its boundary conditions are written as \[\varphi_f(0,1,1)= 1 \textup{ and } {\varphi_f(1,0,0) = 0}.\] We also add the condition: $\varphi_f(0,0,z) = z$. This is a classic property of argumentation, corresponding to the principle of stability (\textbf{A5}): an argument that is neither attacked nor supported has an acceptability equal to its intrinsic weight.

Depending on the choices made for $\varphi_*$ regarding the neutral element (\textbf{P9}) and the null element (\textbf{P10}), other boundary conditions may be necessary. Indeed, out-of-scale values must be reachable. Thus, those supplementary boundary conditions depend on the existence and value of the different neutral and null elements, with one instance of each element per function.

For example, if $\varphi_{\mathcal{R}}$ and $\varphi_{\mathcal{S}}$ both have a neutral element and if these two neutral elements are the same, for example $e_1 \neq 0$, we propose the following: ${\varphi_f(1,1,z) = e_1}$. This condition states that an argument with maximum global weight from both attackers and supporters is useless. Indeed, in this situation, this argument is both completely attacked and completely supported, and should therefore be left undecided. 

Similarly, we can for example choose $\varphi_{\mathcal{R}}$ and $\varphi_{\mathcal{S}}$ such that they have the same null element $e_0$ and $e_0 \neq 1$. Following the discussion in the previous section regarding attackers and supporters of such arguments, it could be required that, in this case, $\varphi_f(0,0,z) = e_0$. However, this condition contradicts the stability principle (\textbf{A5}). Moreover, not being attacked and supported does not necessarily mean that the argument is irrefutable. Thus, this condition does not make sense. This is because irrefutable arguments are a specific type of argument. One way to model this would be to define a specific value for $w(a)$, but we leave this discussion for future work.

\begin{figure}[t!]
    \centering
    
    \begin{subfigure}{0.49\textwidth}
    \centering
        \begin{tikzpicture}[scale=0.8,transform shape, node distance={25mm}, main/.style = {draw, circle,minimum size=12mm}]
		\node[main] (1) {$a,0.5$};
		
		\node[main] (2)  at (-1.8,1.4) {$b,0.1$};
		\node[main] (3)  at (1.8,1.4) {$c,0.6$};
		\node[main] (5)  at (-1.8,-1.4) {$e,0.59$};
		\node[main] (4)  at (1.8,-1.4) {$d,0.3$};

		\draw [arrows = {-Latex[scale=1.5]}, dashed] (2) -- (1);
		\draw [arrows = {-Latex[scale=1.5]}] (3) -- (1);
		\draw [arrows = {-Latex[scale=1.5]}] (4) -- (1);
		\draw [arrows = {-Latex[scale=1.5]}] (5) -- (1);
		
	\end{tikzpicture}
    \end{subfigure}%
    \hfill
    \begin{subfigure}{0.49\textwidth}
    \centering
        \begin{tikzpicture}[scale=0.8,transform shape, node distance={25mm}, main/.style = {draw, circle,minimum size=12mm}]
		\node[main] (1) {$a,0.5$};
		
		\node[main] (2)  at (-1.8,1.4) {$b,0.1$};
		\node[main] (3)  at (1.8,1.4) {$c,0.6$};
		\node[main] (5)  at (-1.8,-1.4) {$e,0.6$};
		\node[main] (4)  at (1.8,-1.4) {$d,0.3$};

		\draw [arrows = {-Latex[scale=1.5]}, dashed] (2) -- (1);
		\draw [arrows = {-Latex[scale=1.5]}] (3) -- (1);
		\draw [arrows = {-Latex[scale=1.5]}] (4) -- (1);
		\draw [arrows = {-Latex[scale=1.5]}] (5) -- (1);
		
	\end{tikzpicture}
    \end{subfigure}
    \caption{Illustrative example for the continuity postulate (\textbf{P3}) for $\varphi_f$.}
    \label{fig:continuity_f}
\end{figure}
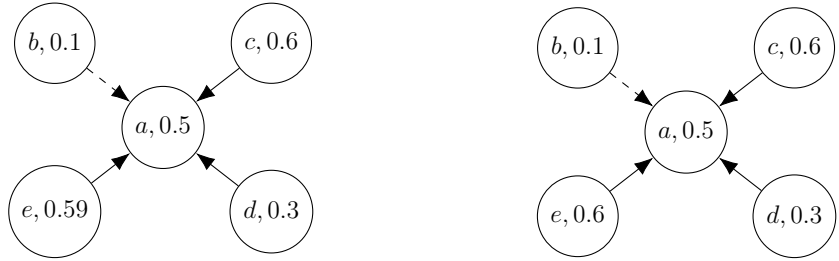

\textbf{(P3): Continuity of $\varphi_{f}$ --} As for $\varphi_{\mathcal{R}}$ and $\varphi_{\mathcal{S}}$, enforcing the continuity of $\varphi_f$ depends on whether threshold and saturation effects are desired. Consider, for example, the case where different pieces of evidence incriminate a suspect in a trial, i.e. attack his or her innocence. As long as a certain number of pieces of evidence of a certain quality have not been presented, i.e. as long as the global weight of the attackers has not reached a certain value, the judge does not consider this suspect to be guilty. Consequently, the various supports on the suspect are taken into account and increase the level of innocence, that is, the degree of acceptability of the argument. However, if enough evidence is presented, the suspect is judged guilty and the acceptability of the argument representing his or her innocence decreases to $0$. In this case, the function $\varphi_f$ is not continuous.

\begin{example} To illustrate the effect of the continuity for $\varphi_f$, we consider the QBAF represented on the left in Figure~\ref{fig:continuity_f}. We take the following function:
\[   \varphi_f(x,y,z) = \left\{
\begin{array}{ll}
      \bot_{\textup{\L uka}}(y,z) = \min(y+z,1) & \textup{if } x < 0.5  \\
      \max\left(0,\frac{1}{3} (-x+y+z)\right) & \textup{otherwise.} 
\end{array} 
\right. \]

\noindent and $\varphi_{\mathcal{R}} = \varphi_{\mathcal{S}} = \textup{avg}_\textup{am}$. We consider two scenarios in which the intrinsic strength of $e$ varies as represented on the left and on the right of Figure~\ref{fig:continuity_f}. 
In the first case, $w(e) = 0.59$ and $\pi_{\mathcal{R}}(a) = 0.497$, hence $\acc(a) = 0.6$. In the second scenario, $w(e) = 0.6$. Consequently,  $\pi_{\mathcal{R}}(a) = 0.5$ and ${\acc(a) = 0.03}$. Thus, a small variation of $\pi_\mathcal{R}(a)$ leads to a big variation of $\acc(a)$. This illustrates the saturation effect that one might want to model the presumption of innocence for instance.

\end{example}



\textbf{(P5): Idempotence of $\varphi_{f}$ --} This property applied to $\varphi_f$ is a special case of a more general one corresponding to a form of symmetry between attackers and supporters: 
\[\forall (x,z) \in J \times I, \varphi_f(x,x,z)=z.\] If this postulate holds, then an argument $a$ being equally attacked and supported, i.e. $\pi_\mathcal{R}(a) = \pi_\mathcal{S}(a)$, has an acceptability degree equal to its intrinsic weight~$w(a)$: the attackers and supporters of $a$ compensate each other. This is a desired behaviour when they both share the same importance in the acceptability of an argument. Note that this does not mean that they play a symmetrical role since $\varphi_\mathcal{R}$ and $\varphi_\mathcal{S}$ can be different.

\textbf{(P7-8): Weakening and reinforcement of $\varphi_f$ --} Due to the monotony of~$\varphi_f$, the weakening postulate becomes, for given meanings of high and low: the combination of a high global weight of attackers with a low weight of supporters and a low intrinsic weight for the argument provides an even lower acceptability. Similarly, an argument that is strong, weakly attacked and strongly supported has a reinforced acceptability value, i.e. even stronger. 
If $\varphi_f$ has a hybrid behaviour, i.e. verifies these two postulates, then the acceptability values are spread, leading to more contrast between strong and weak arguments.


\textbf{(P10): Null element of $\varphi_{f}$ --} The meaning of a null element for $\varphi_f$ depends on its position, resulting in different expected outputs. More precisely, a null element for the first variable $x$ corresponds to being attacked so much, respectively so less, that no matter how supported and intrinsically strong the argument is, the aggregated value would remain the same, for example $0$, respectively $1$. With a symmetrical yet similar meaning for the second variable $y$, a conflict would arise if $x$ and $y$ are both equal to their respective null element. This would create a new boundary condition, for which a value has to be chosen in case this conflict were to happen. 

Finally, a null element for the third variable $z$ would correspond to special arguments. Its meaning could be diverse, ranging from irrefutable argument (see the discussion about \textbf{P10} for $\varphi_*$), to arguments whose intrinsic strength has not been evaluated. As a consequence, the choice of the existence of a null element strongly depends on the application.

\paragraph{Partial conclusion} The monotony postulate (\textbf{P2}) of $\varphi_f$ is expected in any context. Continuity (\textbf{P3}), idempotence (\textbf{P5}), weakening (\textbf{P7}), strengthening (\textbf{P8}) and the existence of a null element (\textbf{P10}) strongly depend on the context. Finally, the boundary conditions~(\textbf{P1}) are expected in any context, but they also depend on the choices of $\varphi_{\mathcal{R}}$ and $\varphi_{\mathcal{S}}$.

\section{A Comparison with the Principles of the Gradual Semantics}
\label{sec:principles}

Section~\ref{sec:discu_agreg}, this section proposes a formal comparison of the principles (\textbf{A1-12}) used in the literature, as well as some commonly used gradual semantics, with respect to aggregative semantics. In particular, we examine the QBAF principles introduced by~\citet{amgoud2018weighted}, whose formal definition is recalled in Section~\ref{sec:ax_aaf}, and discuss how they relate to the postulates detailed in Section~\ref{sec:discu_agreg}. We distinguish between two types of principles: those that are formally related to the postulates of the aggregation functions, for which we give a sufficient condition, and those that describe properties involving the three aggregation functions, imposing a joint choice of~$\varphi_{\mathcal{R}},\varphi_{\mathcal{S}}$ and~$\varphi_f$. For this second category, we propose a reformulation for aggregative semantics. The proofs of all propositions are straightforward and therefore only one of them is given in this section while the others can be found in the appendix. 

\subsection{Relation between Argumentative Principles and Postulates of Aggregation Operators}

The proposed aggregation semantics defines a particular structure for calculating the degree of acceptability of an argument. Indeed, only the degrees of acceptability of the direct attackers and direct supporters of an argument and its own weight are taken into account in this calculation. For this reason, some principles of argumentation are verified without any additional hypothesis. Throughout this section $A = (\mathcal{A},\mathcal{R},\mathcal{S},w)$ is an acyclic QBAF and $\accs$ is an aggregative semantics, whose aggregation function are $\varphi_\mathcal{R},\varphi_\mathcal{S},\varphi_f$.

\begin{proposition}
	\label{th:no_hyp}
Let $\accs$ be an aggregative semantics defined in Definition~\ref{def:sem_agg}. Then, $\accs$ satisfies the following principles: anonymity (\textbf{A1}), independence (\textbf{A2}), directionality (\textbf{A3}) and equivalence (\textbf{A4}).
\end{proposition}

\begin{proof}

    Let ${A},{A}' \in \textup{ac-WAG}$.
	
	\textbf{(A1)}: If $A$ and $A'$ are isomorph, let $f$ be an isomorphism between ${A}$ and ${A}'$.
	
	Let $a \in \mathcal{A}$ be an argument, let us prove that:  $\accsprime(f(a))= \accs(a)$. 
	
	By definition: $\accsprime(f(a)) = \varphi_f(\pi_{\mathcal{R}}(f(a)),\pi_{\mathcal{S}}(f(a)),w'(f(a)))$.
	
     First, let us consider the case where $\text{Supp}_a = \text{Att}_a = \emptyset$.
	
	As $f$ is an isomorphism, $\text{Supp}_{f(a)} = \text{Att}_{f(a)} = \emptyset$ and $w'(f(a)) = w(a)$.
	
	Therefore: \begin{align*}
	    \varphi_f(\pi_{\mathcal{R}}(f(a)),\pi_{\mathcal{S}}(f(a)),w'(f(a))) &= \varphi_f(0,0,w(a))\\
        &= \varphi_f(\pi_{\mathcal{R}}(a),\pi_{\mathcal{S}}(a),w(a)) \\
        &= \accs(a).
	\end{align*}

	Let us now denote $\tilde{\mathcal{A}} = \{a \in \mathcal{A} \mid \accs(a) = \accsprime(f(a))\}$. 
	
	As $A$ is acyclic, $\exists a \in \mathcal{A}$ such that $\text{Supp}_a = \text{Att}_a = \emptyset$. Therefore, $\tilde{\mathcal{A}} \neq \emptyset$. 
	
	Let us prove that $\tilde{\mathcal{A}} = \mathcal{A}$. 
    By definition $\tilde{\mathcal{A}} \subseteq \mathcal{A}$, let us prove by contradiction that $\tilde{\mathcal{A}} \supseteq \mathcal{A}$: let us suppose that $\exists a \in \mathcal{A}\setminus \tilde{\mathcal{A}}$. Two cases may occur:
        \begin{enumerate}
            \item[(i)]  $\text{Att}_a \subseteq \tilde{\mathcal{A}}$ and $\text{Supp}_a \subseteq \tilde{\mathcal{A}}$
            \item[(ii)]  $\text{Att}_a \not\subseteq \tilde{\mathcal{A}}$ or $\text{Supp}_a \not\subseteq \tilde{\mathcal{A}}$
        \end{enumerate}
		
		(i) If $\text{Att}_a \subseteq \tilde{\mathcal{A}}$ and $\text{Supp}_a \subseteq \tilde{\mathcal{A}}$, then by definition: \begin{align*}
	    \accsprime(f(a)) &= \varphi_f(\pi_{\mathcal{R}}(f(a)),\pi_{\mathcal{S}}(f(a)),w'(f(a)))\\
        &= \varphi_f(\pi_{\mathcal{R}}(f(a)),\pi_{\mathcal{S}}(f(a)),w(a)).
	\end{align*}  
		
		As $ \forall b_i \in \text{Att}_a, b_i \in \tilde{\mathcal{A}}$ then, $\accsprime(f(b_i)) = \accs(b_i)$. 
		
		Therefore: \begin{align*}
	    &\varphi_{\mathcal{R}}(\accsprime(f(b_1)), \dotsc, \accsprime(f(b_n)))=~\varphi_{\mathcal{R}}(\accs(b_1),\dotsc,\accs(b_n)).
	\end{align*}
    
        
        Hence, $\pi_{\mathcal{R}}(f(a)) = \pi_{\mathcal{R}}(a)$. 
		
		Similarly, $\pi_{\mathcal{S}}(f(a)) = \pi_{\mathcal{S}}(a)$.
		
		Consequently, $a \in \tilde{\mathcal{A}}$ which is a contradiction.

        (ii) If $\text{Att}_a \not\subseteq \tilde{\mathcal{A}}$ or $\text{Supp}_a \not\subseteq \tilde{\mathcal{A}}$, let us first consider the case $\text{Att}_a \not\subseteq \tilde{\mathcal{A}}$ and denote $a_1 \in \text{Att}_{a} \setminus \tilde{\mathcal{A}}$. 

        $a_1 \not \in \tilde{\mathcal{A}}$ implies $\text{Supp}_{a_1} \not\subseteq \tilde{\mathcal{A}}$ or $\text{Att}_{a_1} \not\subseteq \tilde{\mathcal{A}}$.
			
		Let us suppose that $\text{Att}_{a_1} \not\subseteq \tilde{\mathcal{A}}$, then $\exists a_2 \in \text{Att}_{a_1} \setminus \tilde{\mathcal{A}}$. 
			
        Because the graph is acyclic and $\mathcal{A}$ is finite, applying this reasoning at most~$n$ times, where $n$ is the number of arguments, leads to an argument $a_0$ such that $\text{Supp}_{a_0} = \text{Att}_{a_0} = \emptyset$ and $a_0 \in \tilde{\mathcal{A}}$. This is a contradiction.

        A similar contradiction can be obtained if $\text{Supp}_a \not\subseteq \tilde{\mathcal{A}}$.

	Therefore, we proved by contradiction that $\mathcal{A}= \tilde{\mathcal{A}}$.
	\smallskip
	
	\textbf{(A2)}: If $\mathcal{A} \cap \mathcal{A'} = \emptyset$ then $\forall a \in \mathcal{A}, \text{Att}_a^{A} = \text{Att}_a^{A\oplus A'}$ and $\text{Supp}_a^{A} = \text{Supp}_a^{A\oplus A'}$.
	
	Hence, $\pi_{\mathcal{R}}^{A}(a) = \pi_{\mathcal{R}}^{A \oplus A'}(a)$ and $\pi_{\mathcal{S}}^{A}(a) = \pi_{\mathcal{S}}^{A \oplus A'}(a)$.
	
	As $w^A(a) = w^{A \oplus A'}(a)$, then: \begin{align*}
	    \accs(a) &= \varphi_f(\pi_{\mathcal{R}}^{A}(a), \pi_{\mathcal{S}}^{A}(a), w^A(a))\\
        &=\varphi_f(\pi_{\mathcal{R}}^{A \oplus A'}(a), \pi_{\mathcal{S}}^{A \oplus A'}(a), w^{A \oplus A'}(a)) = \customaccbis{{A}\oplus A'}{S}(a).
	\end{align*}
    
	\smallskip
	
	\textbf{(A3)}: Let us assume that $\mathcal{A} = \mathcal{A'}, w = w', \mathcal{R} \subseteq \mathcal{R}'$ and~$\mathcal{S} \subseteq \mathcal{S}'$.
	
	Let  $a,b,x \in \mathcal{A}$ be arguments, such that $\mathcal{R}' \cup \mathcal{S}' = \mathcal{R} \cup\mathcal{S}\cup \{(a,b)\}$ and there is no path between $b$ and $x$. The latter condition implies that $b \notin \text{Att}_x \cup \text{Supp}_x$.
	
	Therefore, $\text{Att}_x^{A} = \text{Att}_x^{A'}$ and $\text{Supp}_x^{A} = \text{Supp}_x^{A'}$.

	Hence, $\pi_{\mathcal{R}}^{A}(x) = \pi_{\mathcal{R}}^{A'}(x)$ and $\pi_{\mathcal{S}}^{A}(x) = \pi_{\mathcal{S}}^{A'}(x)$.
	
	As $w(x) = w'(x)$, then:\begin{align*}
	    \accs(x) &= \varphi_f(\pi_{\mathcal{R}}^{A}(x), \pi_{\mathcal{S}}^{A}(x), w(x))\\
        &=\varphi_f(\pi_{\mathcal{R}}^{A'}(x), \pi_{\mathcal{S}}^{A'}(x), w'(x))\\
        &= \accsprime(x).
	\end{align*}
    
	\smallskip
	
	\textbf{(A4)}: Let $a,b \in \mathcal{A}$ be arguments such that: \begin{itemize}
		\item $w(a) = w(b)$ ,
		\item $\exists f: \text{Att}_a \to \text{Att}_b$ bijective such that: $\forall x \in \text{Att}_a, \accs(x) = \accs(f(x))$,
		\item $\exists f': \text{Supp}_a \to \text{Supp}_b$ bijective such that: $\forall y \in \text{Supp}_a, \accs(y) = \accs(f'(y))$.
	\end{itemize} 
	
	For the sake of legibility, we note $\pi_{\mathcal{R}}(f(a)) = \varphi_{\mathcal{R}}(\{\accs(f(x)) \mid x \in \text{Att}_a\}$ and similarly for $f'$ and for $\varphi_{\mathcal{S}}$.
	
	Since $\forall x \in \text{Att}_a, \accs(x) = \accs(f(x))$ and $\varphi_{\mathcal{R}}$ is commutative (see Definition~\ref{def:sem_agg}), then ${\pi_{\mathcal{R}}(a) = \pi_{\mathcal{R}}(f(a)) = \pi_{\mathcal{R}}(b)}.$
	
	Similarly, $\pi_{\mathcal{S}}(a) = \pi_{\mathcal{S}}(f'(a)) = \pi_{\mathcal{S}}(b)$.
	
	Hence: \begin{align*}
	    \accs(a) &= \varphi_f(\pi_{\mathcal{R}}(a), \pi_{\mathcal{S}}(a), w(a))\\
        &= \varphi_f(\pi_{\mathcal{R}}(b), \pi_{\mathcal{S}}(b), w(b))\\
        &=\accs(b).
	\end{align*}
	
\end{proof}

For the stability principle (\textbf{A5}) to hold, we need to impose additional constraints to $\accs$. A sufficient condition is the following: 
\begin{proposition}[Stability]
	\label{th:equiv}
Let $\acc$ be an aggregative semantics defined in Definition~\ref{def:sem_agg}. If $\varphi_\mathcal{R},\varphi_\mathcal{S}$ and $\varphi_f$ verify the boundary conditions discussed in Section~\ref{sec:discu_agreg}, then $\accs$ verifies the stability principle (\textbf{A5}).
\end{proposition}


The neutrality principle (\textbf{A6}) deals with the effect of attackers and supporters with degree of acceptability equal to zero. They are qualified as useless~\citep{amgoud2018weighted} and therefore must have no effect on the acceptability of an argument: with an aggregative semantics, they must not change the value of the weight of the attackers or supporters. This is not the case by default, and it requires imposing that $0$ is a neutral element for $\varphi_{\mathcal{R}}$ and~$\varphi_{\mathcal{S}}$, at the cost of confusing useless arguments and inacceptable ones.

\begin{proposition}[Neutrality]\label{prop:qbaf_neutre}
Let $\accs$ be an aggregative semantics defined in Definition~\ref{def:sem_agg}. If $0$ is a neutral element of $\varphi_{\mathcal{R}}$ and $\varphi_{\mathcal{S}}$, then $\accs$ verifies the neutrality principle (\textbf{A6}).
\end{proposition}


The bivariate monotony (\textbf{A7}) and bivariate reinforcement~(\textbf{A8}) principles concern the evolution of the degree of acceptability when attackers or supporters are added (monotony) or when their acceptability is changed (reinforcement).

\begin{proposition}[Reinforcement]\label{prop:qbaf_renf}
Let $\accs$ be an aggregative semantics defined in Definition~\ref{def:sem_agg}. If $\varphi_{\mathcal{R}}$ and $\varphi_{\mathcal{S}}$ are increasing (resp. strictly increasing) and $\varphi_f$ is decreasing in $x$ and increasing in $y$ (resp. strictly monotonous in $x$ and $y$), $\accs$ satisfies the reinforcement principle (resp. strict reinforcement)~(\textbf{A8}).
\end{proposition}


For the monotony principle (\textbf{A7}), an additional condition is needed. Indeed, there is no postulate for aggregation functions that makes explicit the behaviour when an argument is added. Indeed, the composition postulate (\textbf{P9}) concerns the addition of the same set of arguments to two sets; the weakening and strengthening postulates (\textbf{P7-8}) give a lower or upper bound on the result. Thus, none of them gives any information about the evolution compared to the absence of an argument in the set of attackers. However, the following proposition can be established:

\begin{proposition}[Monotony]\label{prop:qbaf_monotonie}
Let $\accs$ be an aggregative semantics defined in Definition~\ref{def:sem_agg}. If $0$ is a neutral element for~$\varphi_{\mathcal{R}}$ and $\varphi_{\mathcal{S}}$, and if the three aggregation functions are monotonous (resp. strictly monotonous), then  $\accs$ verifies the monotony (resp. strict monotony) principle~(\textbf{A7}).
\end{proposition}

\begin{proof}[Sketch of the proof]
	Since $0$ is a neutral element and $\varphi_{*}$ is increasing, adding an attacker (or supporter) cannot lower the total weight of the attackers (or the supporters). Furthermore, if the functions are strictly monotonous, this value can only increase: \[\forall z \in [0,1], \varphi_\mathcal{R}(\mathbf{x}_n) = \varphi_\mathcal{R}(\mathbf{x}_n, 0) \leq \varphi_\mathcal{R}(\mathbf{x}_n,z).\] Given this property, the proof is straightforward by writing the inequalities between the different functions.
\end{proof}

The resilience principle (\textbf{A9}) gives a special role to the extreme values of the intrinsic strength and acceptability of an argument. The maximum and minimum degrees of acceptability can only be reached by arguments that already have this value as their intrinsic strength. Thus, an argument with an intrinsic strength other than~$1$ can never be fully accepted, regardless of the quality of its supporters. This is an optional principle, the relevance of which depends on the nature of the arguments~\citep{amgoud2018weighted}. A reformulation of this principle for aggregative semantics is fairly straightforward: \[\forall z \not\in \{0,1\}, \forall x,y \in [0,1]^2, 0<\varphi_f(x,y,z)<1.\]

In order to get most of the principles, very few postulates, in terms of quantity and diversity, are required. Therefore, the other postulates are complementary.

\subsection{About the Interactions between the three Aggregation Operators}

The last three principles have no connection with the postulates of aggregation functions. In particular, unlike these local properties, the principles 
(\textbf{A10}), (\textbf{A11}) and (\textbf{A12}) are global properties that give information about the interactions between $\varphi_{\mathcal{R}}$, $\varphi_{\mathcal{S}}$ and~$\varphi_f$, such as the desired behaviour when an attacker and a supporter of equal acceptability are added simultaneously. In order to better guide the choice of the three aggregation functions, we propose to rewrite these three principles in the formalism of aggregative semantics.

\paragraph{Franklin's principle (\textbf{A10})} It indicates how the attackers and supporters of an argument should be aggregated. In particular, this principle governs situations in which an attacker and a supporter of equal acceptability are added and more precisely considers that the attacker dominates the supporter and therefore the acceptability decreases. In the strict case, the attacker and the supporter compensate each other, 
leading to an unchanged acceptability value.
In the proposed aggregative formalism, where asymmetry is allowed by independent choices for $\varphi_{\mathcal{R}}$ and $\varphi_{\mathcal{S}}$, the addition of an attacker and a supporter of equal strength does not necessarily have the same effect on the global weight of attackers and supporters. Similarly, if the global weight of the attacker is greater than the global weight of the supporters, it does not imply that the induced acceptability degree is lower than the intrinsic weight of an argument. Therefore, this principle clearly controls the overall behaviour of the semantics. Rewriting this principle within the aggregative formalism leads to
\[{\varphi_f(\varphi_{\mathcal{R}}(\mathbf{x_n},x),\varphi_{\mathcal{S}}(\mathbf{y_m},x),z) \leq \varphi_f(\varphi_{\mathcal{R}}(\mathbf{x_n}),\varphi_{\mathcal{S}}(\mathbf{y_m}),z)}\] and equality in the strict case. 

\paragraph{Weakening (\textbf{A11})} Informally, (\textbf{A11}) deals with the case where the attackers dominate the supporters, i.e., basically, for each supporter, there is an attacker with a greater degree of acceptability. In this situation, the degree of acceptability of the argument must be not greater than its intrinsic weight. Using the aggregative formalism with the notation $\pi_{\mathcal{R}}$ and $\pi_{\mathcal{S}}$, \textbf{A11} may be written, under the same conditions, as:

\[
\varphi_f(\pi_{\mathcal{R}}(a),\pi_{\mathcal{S}}(a),w(a)) < w(a).
\]

\noindent This is again a property that binds the three aggregation functions together. 

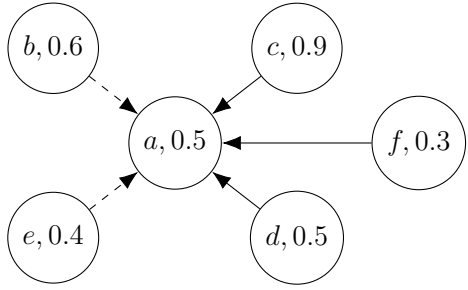
\begin{figure}[t!]
    
    \centering
        \begin{tikzpicture}[scale=0.9,transform shape, node distance={25mm}, main/.style = {draw, circle,minimum size=12mm}]
		\node[main] (1) {$a,0.5$};
		
		\node[main] (2)  at (-1.8,1.4) {$b,0.6$};
		\node[main] (3)  at (1.8,1.4) {$c,0.9$};
		\node[main] (5)  at (-1.8,-1.4) {$e,0.4$};
		\node[main] (4)  at (1.8,-1.4) {$d,0.5$};
		\node[main] (6)  at (3.6,0) {$f,0.3$};

		\draw [arrows = {-Latex[scale=1.5]}, dashed] (2) -- (1);
		\draw [arrows = {-Latex[scale=1.5]}] (3) -- (1);
        \draw [arrows = {-Latex[scale=1.5]}] (6) -- (1);
		\draw [arrows = {-Latex[scale=1.5]}] (4) -- (1);
		\draw [arrows = {-Latex[scale=1.5]}, dashed] (5) -- (1);
		
	\end{tikzpicture}
    \caption{Illustrative example of the entanglement of the weakening axiom (\textbf{A11}).}
    \label{fig:ax_weak}
\end{figure}%

\begin{example} To illustrate the entanglement between the three aggregation function that the weakening axiom imposes, let us consider the QBAF represented in Figure~\ref{fig:ax_weak}.

Let us take $\varphi_{\mathcal{R}} =  \varphi_{\mathcal{S}} = \min$ and $$\varphi_f(x,y,z) = \min\left(1,(1+y)\max(0,z-x)\right).$$ 

With the induced aggregative semantics, it is possible to define an injective function $F$ from $\textup{Supp}_a$ to $\textup{Att}_a$ such that $\forall x \in  \textup{Supp}_a$, $\acc(x) < \acc(F(x))$. Indeed, one can take $F$ such that $F(b) = c$ and $F(e) = d$. Hence, the attackers of~$a$ ``dominate'' the supporters of $b$.

However, $\pi_{\mathcal{R}}(a) = 0.3$ and $\pi_{\mathcal{S}}(a) = 0.4$, i.e. $\pi_{\mathcal{R}}(a) < \pi_{\mathcal{S}}(a)$. Thus, even though the attackers ``dominate'' the supporters, the order between their respective global strength is reversed because $\acc(f)$ is ``weak'' and $f\in \textup{Att}_a$.

Now, regarding the acceptability of argument $a$, $\acc(a) = \varphi_f(0.3,0.4,0.5) = 0.28 < w(a) = 0.5$. Thus, the weakening principle holds for this aggregative semantics even if $\pi_{\mathcal{R}}(a) < \pi_{\mathcal{S}}(a)$.

\end{example}

\paragraph{Strengthening (\textbf{A12})} Similarly, the principle of strengthening deals with the case where supporters dominate attackers, and requires that the acceptability of the argument is strictly greater than its intrinsic weight. Keeping the same hypotheses, it is expressed in the formalism of aggregative semantics as follows:
\[
\varphi_f(\pi_{\mathcal{R}}(a),\pi_{\mathcal{S}}(a),w(a)) > w(a).
\]

These three principles define how gradual semantics behave at a global level with respect to sets of attackers and supporters. In the case of aggregative semantics, these properties are complementary to the various postulates discussed above, which govern the behaviour of each function at a local level. Combining these two level properties to improve the granularity of the design of gradual semantics leads to a joint selection of these three functions rather than an independent one.

\subsection{Instantiation of some Gradual Semantics}
\label{sec:discu}

In this section we consider the three gradual semantics commonly used in the literature, DF-Quad, Ebs and QE, whose rewriting in the formalism of aggregative semantics is given in the proof of Proposition~\ref{prop:mod_to_agg}. We discuss the postulates verified by their implementation, summarised in Table~\ref{tab:ex_fonc}.

\paragraph{Postulates verified by $\varphi_{\mathcal{R}}/\varphi_{\mathcal{S}}$} A first remark concerns the domain of definition of the functions used: throughout the article we considered ${I=J=[0,1]}$. In the case of the sum function used for $\varphi_\mathcal{R}$ and $\varphi_\mathcal{S}$ by Ebs and QE, $I = \mathbb{R}$ which does not allow the boundary conditions \textbf{(P1)}. They are therefore replaced by limits in~$\pm \infty$. Regarding the other postulates for calculating the global weight of attackers and supporters, the only differences concern the existence of a null element (\textbf{P10}): the sum does not have a null element while the algebraic sum has one. These two postulate lead to a difference in the behaviour of those functions. Both of these functions satisfy the reinforcement postulate (\textbf{P9}), so the more added arguments, the better. However, unlike the algebraic sum $\bot_\text{prod}$, the sum function models the idea that the quantity of arguments matters more than their individual strength. For example, it is better to be attacked or supported by six arguments with an acceptability degree of $0.2$ than by one argument with maximum acceptability, i.e. $1$, whereas the algebraic sum induces the reverse behaviour.

In the light of the discussion in Section~\ref{sec:discu_agreg}, some important postulates are absent. First, all these functions are commutative \textbf{(P4)}, so none of them can be used in a meaningful way in contexts where the order of the arguments has to be taken into account. Secondly, they are disjunctive operators, i.e. they satisfy the reinforcement postulate (\textbf{P8}) and the weakening one (\textbf{P7}) does not hold. Therefore, if one wants to adopt a compromise or a pessimistic behaviour, neither the sum nor the algebraic sum can be used. Finally, we observe in the literature an essentially symmetrical management of attackers and supporters for the calculation of a global weight of the relations of an argument.
\renewcommand{\arraystretch}{1.3}
\begin{table}[t]
    \centering
    \caption{Verified ($\checkmark$), violated ($\times$) and inapplicable (-) postulates for the aggregation operators used in the three commonly used gradual semantics of the literature. }
    \vspace{2mm}
	\resizebox{0.95\linewidth}{!}{\Large\begin{tabular}{|c|cccccccccccc|}
			\hline
			\multicolumn{1}{|c|}{$\varphi_{\mathcal{R}}/\varphi_{\mathcal{S}}$} & \multicolumn{1}{c|}{(P1)} & \multicolumn{1}{c|}{(P2)} & \multicolumn{1}{c|}{(P3)} & \multicolumn{1}{c|}{(P4)} & \multicolumn{1}{c|}{(P5)} & \multicolumn{1}{c|}{(P6)} & \multicolumn{1}{c|}{(P7)} & \multicolumn{1}{c|}{(P8)} & \multicolumn{1}{c|}{(P9)} & \multicolumn{1}{c|}{(P10)} & \multicolumn{1}{c|}{(P11)} & (P12) \\ \hline
			\multicolumn{1}{|c|}{Sum} & \multicolumn{1}{c|}{$\times$} & \multicolumn{1}{c|}{$\checkmark$} & \multicolumn{1}{c|}{$\checkmark$} & \multicolumn{1}{c|}{$\checkmark$} & \multicolumn{1}{c|}{$\times$} & \multicolumn{1}{c|}{$\checkmark$} & \multicolumn{1}{c|}{$\times$} & \multicolumn{1}{c|}{$\checkmark$} & \multicolumn{1}{c|}{$\checkmark$} &\multicolumn{1}{c|}{$\times$} & \multicolumn{1}{c|}{$\checkmark$} & $\checkmark$ \\ \hline
			\multicolumn{1}{|c|}{$\bot_\textup{prod}$} &\multicolumn{1}{c|}{$\checkmark$} & \multicolumn{1}{c|}{$\checkmark$} & \multicolumn{1}{c|}{$\checkmark$} & \multicolumn{1}{c|}{$\checkmark$} & \multicolumn{1}{c|}{$\times$} & \multicolumn{1}{c|}{$\checkmark$} & \multicolumn{1}{c|}{$\times$} & \multicolumn{1}{c|}{$\checkmark$} & \multicolumn{1}{c|}{$\checkmark$} &\multicolumn{1}{c|}{$\checkmark$}& \multicolumn{1}{c|}{$\checkmark$} & $\times$ \\ \hline \hline
			\multicolumn{1}{|c|}{$\varphi_f^{\textup{DF-Quad}}$} & \multicolumn{1}{c|}{$\checkmark$} & \multicolumn{1}{c|}{$\checkmark$} & \multicolumn{1}{c|}{$\checkmark$} & \multicolumn{1}{c|}{-} & \multicolumn{1}{c|}{$\checkmark$} & \multicolumn{1}{c|}{-} & \multicolumn{1}{c|}{$\times$} & \multicolumn{1}{c|}{$\times$} & \multicolumn{1}{c|}{-} &\multicolumn{1}{c|}{$\times$}& \multicolumn{1}{c|}{-} & - \\ \hline
			\multicolumn{1}{|c|}{$\varphi_f^{\textup{Ebs}}$} & \multicolumn{1}{c|}{$\checkmark$} & \multicolumn{1}{c|}{$\checkmark$} & \multicolumn{1}{c|}{$\checkmark$} & \multicolumn{1}{c|}{-} & \multicolumn{1}{c|}{$\checkmark$} & \multicolumn{1}{c|}{-} & \multicolumn{1}{c|}{$\times$} & \multicolumn{1}{c|}{$\times$}& \multicolumn{1}{c|}{-} &\multicolumn{1}{c|}{$\checkmark$} & \multicolumn{1}{c|}{-} & - \\ \hline
			\multicolumn{1}{|c|}{$\varphi_f^{\textup{QE}}$} & \multicolumn{1}{c|}{$\checkmark$} & \multicolumn{1}{c|}{$\checkmark$} & \multicolumn{1}{c|}{$\checkmark$} & \multicolumn{1}{c|}{-} & \multicolumn{1}{c|}{$\checkmark$} & \multicolumn{1}{c|}{-} & \multicolumn{1}{c|}{$\times$} & \multicolumn{1}{c|}{$\times$}& \multicolumn{1}{c|}{-} &\multicolumn{1}{c|}{$\times$} & \multicolumn{1}{c|}{-} & - \\ \hline
	\end{tabular}}
	\label{tab:ex_fonc}
\end{table}

\paragraph{Postulates verified by $\varphi_{f}$} Except for the existence of a null element (\textbf{P10}), there is no diversity in the properties verified by the three considered functions. In particular, regarding the three postulates for which there is a choice (idempotence, weakening and strengthening), all are idempotent (and verify ${\varphi_f(x,x,z) = z}$), and satisfy neither the weakening nor the strengthening principles. The idempotence (\textbf{P5}) satisfaction is due to the fact that a preliminary step is taken to calculate the difference between the weights of attackers and supporters, and that $0$ is a neutral element of the corresponding influence function. 

However, this lack of diversity does not mean that those postulates are trivial as they are not obviously violated nor satisfied. As an example, let us consider the following functions. First, $\varphi_f(x,y,z)= \min(1-x,y,z)$. It is an adapted version of the minimum operator.  Similarly for the maximum, ${\varphi_f(x,y,z)= \max(1-x,y,z)}$. Finally, $\varphi_f$ could behave differently depending on the interval, taking the minimum or the maximum if the values are extreme and taking an arithmetic mean otherwise: \[   \varphi_f(x,y,z) = \left\{
\begin{array}{ll}
      \max(1-x,y,z) & \textup{if } x<0.2 \textup{ and } y,z>0.8 \\
      \min(1-x,y,z) & \textup{if } x>0.8 \textup{ and } y,z<0.2 \\
      \frac{1 - x + y + z}{3} & \textup{otherwise.} 
\end{array} 
\right. \]
The three functions verify \textbf{P1}, \textbf{P2} and \textbf{P3}. The minimum verifies \textbf{P7}, the maximum \textbf{P8} and the hybrid version verifies both.

\section{Illustrative Example}\label{sec:ex_comp}

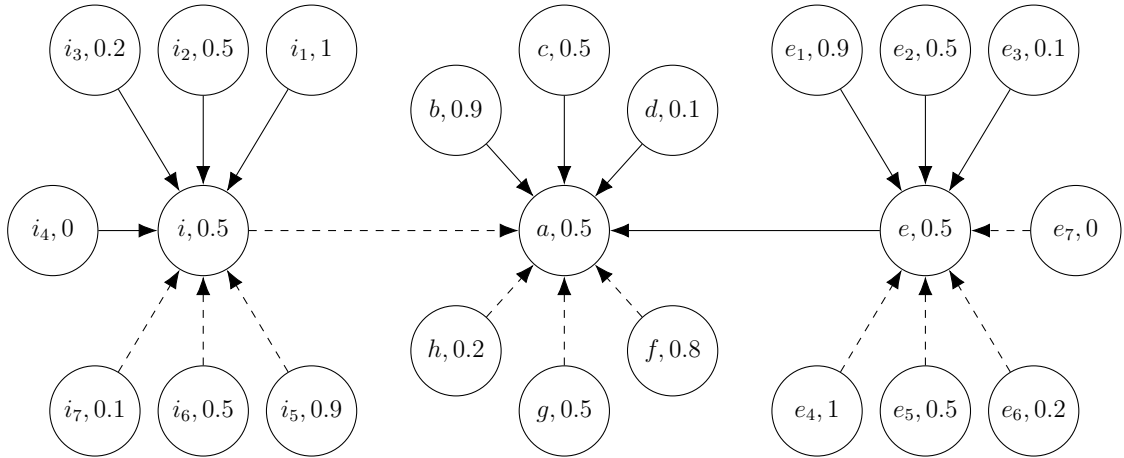
\begin{figure}[t!]
    
    \centering
        \begin{tikzpicture}[scale=0.8,transform shape, node distance={25mm}, main/.style = {draw, circle,minimum size=15mm}]
		\node[main] (1) {$a,0.5$};
		
		\node[main] (8)  at (0,3) {$c,0.5$};
		\node[main] (3)  at (1.8,2) {$d,0.1$};
		\node[main] (6)  at (6,0) {$e,0.5$};
		\node[main] (4)  at (1.8,-2) {$f,0.8$};
        
		\node[main] (9)  at (0,-3) {$g,0.5$};
		\node[main] (5)  at (-1.8,-2) {$h,0.2$};
		\node[main] (7)  at (-6,0) {$i,0.5$};
		\node[main] (2)  at (-1.8,2) {$b,0.9$};

        \node[main] (611)  at (4.2,3) {$e_1,0.9$};
        \node[main] (612)  at (6,3) {$e_2,0.5$};
        \node[main] (613)  at (7.8,3) {$e_3,0.1$};
        \node[main] (621)  at (4.2,-3) {$e_4,1$};
        \node[main] (622)  at (6,-3) {$e_5,0.5$};
        \node[main] (623)  at (7.8,-3) {$e_6,0.2$};
        \node[main] (624)  at (8.5,0) {$e_7,0$};

        \node[main] (711)  at (-4.2,3) {$i_1,1$};
        \node[main] (712)  at (-6,3) {$i_2,0.5$};
        \node[main] (713)  at (-7.8,3) {$i_3,0.2$};
        \node[main] (714)  at (-8.5,0) {$i_4,0$};
        \node[main] (721)  at (-4.2,-3) {$i_5,0.9$};
        \node[main] (722)  at (-6,-3) {$i_6,0.5$};
        \node[main] (723)  at (-7.8,-3) {$i_7,0.1$};

        
		\draw [arrows = {-Latex[scale=1.5]}] (2) -- (1);
		\draw [arrows = {-Latex[scale=1.5]}] (3) -- (1);
        \draw [arrows = {-Latex[scale=1.5]}] (6) -- (1);
		\draw [arrows = {-Latex[scale=1.5]}, dashed] (4) -- (1);
		\draw [arrows = {-Latex[scale=1.5]}, dashed] (5) -- (1);
		\draw [arrows = {-Latex[scale=1.5]}, dashed] (7) -- (1);
		\draw [arrows = {-Latex[scale=1.5]}] (8) -- (1);
		\draw [arrows = {-Latex[scale=1.5]}, dashed] (9) -- (1);

        \draw [arrows = {-Latex[scale=1.5]}] (611) -- (6);
        \draw [arrows = {-Latex[scale=1.5]}] (612) -- (6);
        \draw [arrows = {-Latex[scale=1.5]}] (613) -- (6);
        
        \draw [arrows = {-Latex[scale=1.5]}, dashed] (621) -- (6);
        \draw [arrows = {-Latex[scale=1.5]}, dashed] (622) -- (6);
        \draw [arrows = {-Latex[scale=1.5]}, dashed] (623) -- (6);
        \draw [arrows = {-Latex[scale=1.5]}, dashed] (624) -- (6);

        \draw [arrows = {-Latex[scale=1.5]}] (711) -- (7);
        \draw [arrows = {-Latex[scale=1.5]}] (712) -- (7);
        \draw [arrows = {-Latex[scale=1.5]}] (713) -- (7);
        \draw [arrows = {-Latex[scale=1.5]}] (714) -- (7);
        \draw [arrows = {-Latex[scale=1.5]}, dashed] (721) -- (7);
        \draw [arrows = {-Latex[scale=1.5]}, dashed] (722) -- (7);
        \draw [arrows = {-Latex[scale=1.5]}, dashed] (723) -- (7);
		
	\end{tikzpicture}
    \caption{Illustrative example analysed in Section~\ref{sec:ex_comp}.}
    \label{fig:ex_final}
\end{figure}%

\begin{figure}
    \centering
    \includegraphics[width=0.9\linewidth]{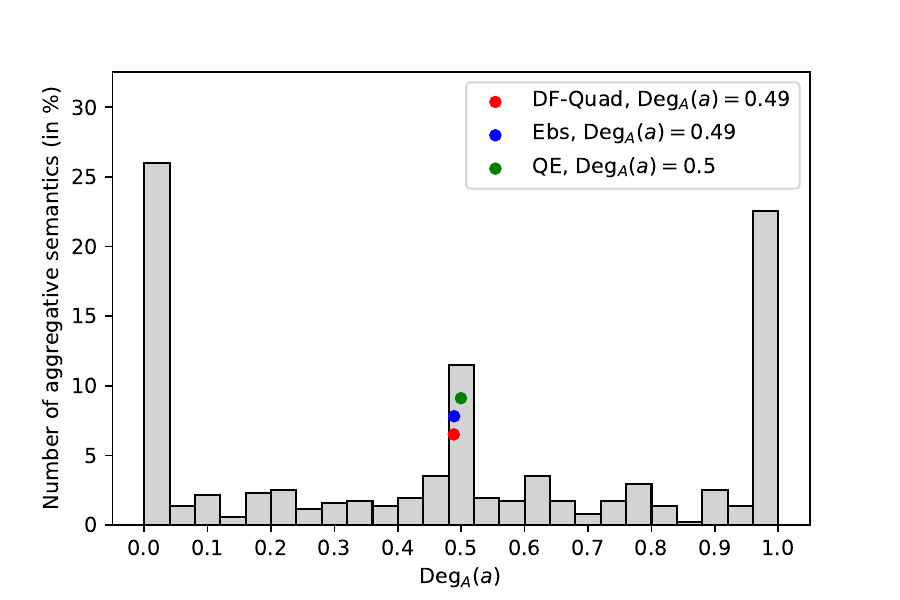}
    \vspace{-2mm}
    \caption{Distribution of acceptability degrees of argument $a$ for the QBAF in Figure~\ref{fig:ex_final}.}
    \label{fig:histo_acc_a}
\end{figure}

In this section we illustrate the diversity of behaviours that can be obtained using aggregative semantics. For that, we use the example represented in Figure~\ref{fig:ex_final}. This QBAF revolves around three central arguments, $i$, $a$ and $e$, that are used to illustrate some specific properties of the aggregative semantics. Argument $a$ is the subject, and its acceptability degree reflects the effect of favouring the supporters or the attackers. Argument $e$ and $i$ have attackers or supporters whose acceptability equals $0$ or $1$, to illustrate the impact of a null element. Finally, there is two injections $F_1: \textup{Att}_i \rightarrow \textup{Supp}_i$ and $F_2: \textup{Supp}_e \rightarrow \textup{Att}_e$, such that $\acc(i_k)\geq \acc(F_1(i_k))$ and $\acc(e_k)\geq \acc(F_2(e_k))$. Therefore, we can observe the behaviour of the aggregative semantics regarding the weakening (\textbf{A11}) and the strengthening (\textbf{A12}) principles.

We use this QBAF to compute the acceptability degree of every argument in the graph for the $512$ aggregative semantics that can be built using the aggregation functions defined in Table~\ref{tab:ex_fonc_aggreg}, except $\top_D$ and $\bot_D$. Indeed, these drastic choices result in an acceptability value of either $0$ or $1$ for a majority of arguments. We also consider the three common gradual semantics: DF-Quad, Ebs, and QE, described in Section~\ref{sec:sem_grad}, leading to $515$ functions in total. The distribution of the obtained values for $\acc(a)$ is shown in the histogram represented in Figure~\ref{fig:histo_acc_a}, with bin size of $0.04$, and is obtained using the code available here~\footnote{\url{https://gitlab.lip6.fr/munro/aggregative-semantics-for-ac-qbafs}}. The three coloured points represent the values returned by the three gradual semantics from the literature: DF-Quad, Ebs and QE. The main conclusions are the following: \begin{itemize}
    \item Each bin can be reached by an aggregative semantics.
    \item The weakening and strengthening postulates (\textbf{P7-8}) for $\varphi_{\mathcal{R}}$ and $\varphi_\mathcal{S}$ determine which side should be favoured.
    \item A null element has a big impact, even inside a same family of functions (e.g. compromise).
\end{itemize}

The first observation concerns the range of values that can be achieved using only the aggregation functions presented in Table~\ref{tab:ex_fonc_aggreg}. In contrast to DF-Quad, Ebs and QE, that produce similar results, there are aggregative semantics for each part of the codomain. Furthermore, the behaviour of these functions can be chosen according to the discussion in Section~\ref{sec:discu_agreg}. It leads to an aggregative semantics that is transparent and can be understood by examining the different postulates they verify. To illustrate it, Table~\ref{tab:res_agg} presents some examples of aggregative semantics and the results they return. 

\textbf{The first three rows (S$_{1-3}$)} show the results when using three gradual semantics from the literature. 

Before analysing the others rows, we remind that, because of the monotonicity postulate (\textbf{P1}) for $\varphi_f$, the functions in Table~\ref{tab:ex_fonc_aggreg} have to be adapted for $\varphi_f(x,y,z)$. This has been illustrated in Example~\ref{ex:agg_sem}. In this section, the results are obtained by computing $\varphi_f(x,y,z) = \varphi(1-x,y,z)$, where $\varphi$ is the chosen aggregation function in Table~\ref{tab:ex_fonc_aggreg}. Consequently, when an argument $a$ is strongly attacked, then a low value is aggregated with the global strength of the supporters of $a$ and $w(a)$.

\begin{table}[t]
\resizebox{\linewidth}{!}{\begin{tabular}{|l|c|c|c|c|c||c|c|c|}
\hline
&$\varphi_{\mathcal{R}}$ & $\varphi_{\mathcal{S}}$ & $\varphi_f$ & $\acc(i)$ & $\acc(e)$ & $\pi_\mathcal{R}(a)$ & $\pi_\mathcal{S}(a)$ & $\acc(a)$ \\ \hline
\small (S$_{\rownumber}$) &$\bot_{\textup{prod}}$ & $\bot_{\textup{prod}}$ & $\varphi_f^{\textup{DF-Quad}}$& $0.48$  &  $0.52$ & $0.98$ & $0.96$  & $0.49$  \\ \hline
\small (S$_{\rownumber}$) &Sum & Sum &    $\varphi_f^{\textup{Ebs}}$    & $0.48$  & $0.53$  & $2.03$ & $1.97$ & $0.49$     \\ \hline
\small (S$_{\rownumber}$) &Sum & Sum &   $\varphi_f^{\textup{QE}}$    &  $0.48$  & $0.52$ & $2.02$ & $1.98$ & $0.50$  \\ \hline \hline
\small (S$_{\rownumber}$) &$\textup{avg}_{\textup{am}}$ & $\textup{avg}_{\textup{am}}$ & $\textup{avg}_{\textup{am}}$ & $0.53$  & $0.48$  & $0.49$ & $0.50$ & $0.50$     \\ \hline
\small (S$_{\rownumber}$) &$\textup{avg}_{\textup{am}}$ & $\textup{avg}_{\textup{am}}$ & $\top_{\textup{prod}}$ & $0.14$  & $0.11$  & $0.40$ & $0.41$ & $0.12$ \\ \hline
\small (S$_{\rownumber}$) &$\textup{avg}_{\textup{am}}$ & $\textup{avg}_{\textup{am}}$ & $\bot_{\textup{prod}}$ & $0.89$  & $0.86$  & $0.59$ & $0.60$ & $0.88$ \\ \hline \hline
\small (S$_{\rownumber}$) &$\textup{avg}_{\textup{am}}$ & $\max$ & $\top_{\textup{prod}}$ & $0.26$  & $0.25$  & $0.48$ & $0.80$ & $0.23$ \\ \hline
\small (S$_{\rownumber}$) &$\bot_{\textup{\L uka}}$ & $\textup{avg}_{\textup{am}}$ & $\top_{\textup{prod}}$ & $0$  & $0$  & $1$ & $0.38$ & $0$ \\ \hline \hline
\small (S$_{\rownumber}$) &$\textup{avg}_{\textup{am}}$ & $\max$ & $\bot_{\textup{prod}}$ & $0.98$  & $1$  & $0.63$ & $0.98$ & $0.99$ \\ \hline
\small (S$_{\rownumber}$) &$\bot_{\textup{\L uka}}$ & $\textup{avg}_{\textup{am}}$ & $\bot_{\textup{prod}}$ & $0.75$  & $0.71$  & $1$ & $0.57$ & $0.78$ \\ \hline \hline
\small (S$_{\rownumber}$) &$\min$ & $\bot_{\textup{prod}}$ & $\textup{avg}_{\textup{am}}$ & $0.82$  & $0.80$  & $0.10$ & $0.99$ & $0.80$ \\ \hline
\small (S$_{\rownumber}$) &$\bot_{\textup{prod}}$ & $\min$ & $\textup{avg}_{\textup{am}}$ & $0.20$  & $0.18$  &$0.96$ & $0.20$ & $0.25$ \\ \hline \hline
\small (S$_{\rownumber}$) &$\top_{\textup{\L uka}}$ & $\top_{\textup{prod}}$ & $\textup{avg}_{\textup{am}}$ & $0.52$  & $0.51$  & $0$ & $0.04$ & $0.51$ \\ \hline
\small (S$_{\rownumber}$) &$\top_{\textup{\L uka}}$ & $\top_{\textup{prod}}$ & $\textup{avg}_{\textup{gm}}$ & $0.28$  & $0$  & $0$ & $0.02$ & $0.22$ \\ \hline \hline
\small (S$_{\rownumber}$) & $\max$ & $\bot_{\textup{prod}}$ & $\textup{avg}_{\textup{am}}$ & $0.49$  & $0.53$  & $0.9$ & $0.96$ & $0.52$ \\ \hline
\end{tabular}}
\caption{The first three rows are respectively DF-Quad, Ebs and QE.}
\label{tab:res_agg}
\end{table}

\textbf{The following three semantics (S$_{4-6}$)} use an arithmetic mean operator ($\textup{avg}_{\textup{am}}$) for both the attackers and the supporters. Taking the same aggregation function allows to understand the effect $\varphi_f$ has on the final result. In the first case (S$_4$), $\textup{avg}_{\textup{am}}$, which is a compromise operator (\textbf{P7} and \textbf{P8} do not hold), has been chosen.  As \textbf{P7} and \textbf{P8} do not hold, the aggregated value is between the lowest and the highest value to aggregate. Therefore, neither the attackers nor the supporters are favoured. Thus, as $a$ is attacked and supported with almost the same global strength, i.e. ${\pi_{\mathcal{R}}(a) \approx \pi_{\mathcal{S}}(a)}$, $\acc(a) \approx 0.5$. 

The second function (S$_5$) uses the product~($\top_{\text{prod}}$), a pessimistic operator. Indeed, this function verifies the weakening postulate (\textbf{P7}) meaning that the result is inferior to the lowest aggregated value. For argument $a$, $\pi_{\mathcal{S}}(a) = 0.38$ so $\acc(a) \leq 0.38$. For an argument to be highly acceptable using this function, it must be weakly attacked, strongly supported, and intrinsically strong.

The third function (S$_6$) uses the dual of the previous one, the algebraic sum~($\bot_{\text{prod}}$). This is an optimistic operator: for an argument to be highly acceptable, it is enough for it to be either weakly attacked, strongly supported, or intrinsically strong. In the example, argument $a$ is strongly supported, i.e. $\pi_\mathcal{S}(a) = 0.63$. It leads to $\acc(a) \geq 0.63$.

\textbf{The following four aggregation functions (S$_{7-10}$)} are a modification of the fifth and sixth one. The final aggregation function remains either a pessimistic operator (S$_{7-8}$) or an optimistic operator (S$_{9-10}$). The difference lies in how the attackers and supporters are aggregated. In (S$_7$) and (S$_9$), we want to favour the supporters: instead of taking a compromise for the global strength of the supporter, we choose the $\max$ function that is optimistic. Therefore, unlike the attackers for whom it is still a compromise, weaker supporters are required to have a high global strength. This change should increase the acceptability of the arguments, as the supporters now have more weight, or at worst it should not make any difference. This leads to $\acc(a)$ going from $0.12$ for (S$_5$) and $0.88$ for (S$_6$) to $0.23$ for (S$_7$) and~$0.99$ for (S$_9$). By contrast, (S$_8$) and (S$_{10}$) favour the attackers because the bounded sum is an optimistic function. As a consequence, $\acc(a)$ becomes $0$ for (S$_8$), respectively $0.78$ for (S$_{10}$).

It is also possible to favour one side by disadvantaging the other. \textbf{The next two semantics (S$_{11-12}$)} combine both. The final aggregation function is $\textup{avg}_{\textup{am}}$, i.e. a compromise. In (S$_{11}$), $\varphi_\mathcal{R}$ is pessimistic and $\varphi_\mathcal{S}$ is optimistic. Hence, we get that $\acc(a) = 0.8$. This can, for example, model a stubborn person who mostly listens to arguments in favour of their own opinion. By contrast, the functions in~(S$_{12}$) can model someone who is more versatile and favours attacking arguments over supporting ones, leading to $\acc(a) = 0.25$.

\textbf{The next two semantics (S$_{13-14}$)} illustrate the impact of other postulates, in particular the existence of a null element. In these two semantics, both the attackers and the supporters are aggregated using a pessimistic function. The final aggregation is then performed using a compromise operator. The only difference lies in the chosen mean function: unlike the arithmetic mean ($\textup{avg}_{\textup{am}}$), the geometric mean ($\textup{avg}_{\textup{gm}}$) has a null element equal to $0$. This can be seen with argument~$e$, for which, in both cases, $\pi_{\mathcal{S}}(e) = 0$. This results in $\acc^{S_{13}}(e) = 0.51$ while $\acc^{S_{14}}(e) = 0$. Therefore, according to the semantics (S$_{14}$), the acceptability of an argument that is fully attacked ($\pi_{\mathcal{R}} = 1$), uselessly supported ($\pi_{\mathcal{S}} = 0$), or intrinsically null ($w = 0$), is unacceptable, regardless of the other values. 

\textbf{The last aggregative semantics (S$_{15}$)} is the only one, except for DF-Quad,Ebs and QE, that does not obviously violate either the weakening or the strengthening principles~{(\textbf{A11-12})}. Indeed, for the other semantics, at least one of those principle does not hold. This can be visualised using arguments $e$ for the strengthening and $i$ for the weakening. For instance, using function (S$_{4}$), there is an injective function $F$ from $\textup{Supp}_e$ to $\textup{Att}_e$ such that $\acc(F(x))\leq \acc(x)$. However, $\acc(e)<w(e)$.  In (S$_{15}$), $\varphi_\mathcal{R}$ and $\varphi_\mathcal{S}$ are both optimistic but not symmetric. Indeed, the $\max$ is the lowest t-conorm meaning that all the other t-conorms are more optimistic than the $\max$. Therefore, the supporter are a bit more favoured. Using $\textup{avg}_{\textup{am}}$ for $\varphi_f$, it leads to $\acc(a) = 0.52 > w(a)$.

\section{Conclusion}

In this paper we have proposed a new gradual semantics family, which we call aggregative semantics. It is based on decomposing the computation of the acceptability degree of an argument into three distinct aggregation stages: computing the global weight of the attackers, the global weight of the supporters, and aggregating these two weights with the intrinsic strength of the argument. This decomposition has several advantages. First, constraining each step of the computation independently leads to a more parametrisable function. Second, by leveraging the large literature on aggregation functions, it allows for the construction of new gradual semantics using the variety of existing aggregation functions. We proposed an axiomatic approach by discussing the desirable postulates and principles for the three functions, depending on the context, and taking into account the specificities of the argumentative setting. This discussion is also a guide to help users to select appropriate aggregation functions depending on the context. Finally, each of these three steps is easily understandable, based on the behaviour of each chosen aggregation function.  This makes the whole process more transparent and interpretable. 

A first perspective aims at taking advantage of this understandable process to generate explanations of the acceptability degree of arguments. Indeed, this work on aggregative semantics paves the way for different levels of explanation: a first one uses the global strengths of the attackers and of the supporters of an argument; a second one goes back to the degree of acceptability of each of the attackers and supporters; and a final one goes up the chain of attackers and supporters of the argument to be explained until reaching the intrinsic strength of the arguments that are not attacked nor supported. Another perspective is to further extend this work by adding weights on the binary relations as well. This can be done keeping the same formalism, modifying the domain of the different aggregation functions as well as the discussion. A third perspective consists in studying the convergence of an aggregative semantics for cyclic QBAFs, by, for example, identifying necessary conditions on the three aggregation functions. 

\bibliographystyle{elsarticle-harv} 
\bibliography{mybibfile}
\newpage

\section{Appendix}

In all this section we assume $A$ to be an acyclic QBAF and $\accs$ an aggregative semantics.

\begin{proposition*}[Stability]
	Let $\accs$ be an aggregative semantics defined in Definition~\ref{def:sem_agg}. If $\varphi_\mathcal{R},\varphi_\mathcal{S}$ and $\varphi_f$ verify the boundary conditions~(\textbf{P1}) discussed in Section~\ref{sec:discu_agreg}, then $\accs$ verifies the stability principle (\textbf{A5}).
\end{proposition*}

\begin{proof}
	Following the discussion regarding the boundary conditions of $\varphi_\mathcal{R}$ and $\varphi_\mathcal{S}$ in Section~\ref{subsec:phi_etoile}, let us suppose that $\varphi_\mathcal{R}(\emptyset) = 0$ and $\varphi_\mathcal{S}(\emptyset) = 0$.

    Then, $\varphi_f(\varphi_\mathcal{R}(\emptyset),\varphi_\mathcal{S}(\emptyset),z) = \varphi_f(0,0,z) = z$ according to the boundary conditions of $\varphi_f$ (as discussed in Section~\ref{subsec:aggreg_final}).
\end{proof}

\begin{proposition*}[Neutrality]
Let $\accs$ be an aggregative semantics defined in Definition~\ref{def:sem_agg}. If $0$ is a neutral element of $\varphi_{\mathcal{R}}$ and $\varphi_{\mathcal{S}}$, then $\accs$ verifies the neutrality principle (\textbf{A6}).
\end{proposition*}

\begin{proof}
	Let $A \in \textup{ac-WAG}$, be an acyclic QBAF and $\accs$ an aggregative semantics.
	
	Let $a,b,x \in \mathcal{A}$ be arguments such that: \begin{itemize}
	    \item $w(a) = w(b)$
        \item $\accs(x) = 0$
        \item $\text{Att}_{a} \subseteq \text{Att}_b$
        \item $\text{Supp}_{a} \subseteq \text{Supp}_b$
        \item $\text{Att}_b \cup \text{Supp}_b = \text{Att}_a \cup \text{Supp}_a \cup \{x\}$.
	\end{itemize} 
	
	By definition: \begin{align*}
	    \accs(b) = \varphi_f(&\varphi_{\mathcal{R}}(\{\accs(c) \mid c \in \text{Att}_b\}),\varphi_{\mathcal{S}}(\{\accs(d) \mid d \in \text{Supp}_b\}),w(b)).
	\end{align*}
	
	First, $w(b) = w(a)$.
	
	Then, $\text{Att}_b \cup \text{Supp}_b = \text{Att}_a \cup \text{Supp}_a \cup \{x\}$ is equivalent to: 
    
    $(1)$ Either $\text{Att}_b = \text{Att}_a \cup \{x\}$ and $\text{Supp}_b = \text{Supp}_a$
    
    $(2)$ Or $\text{Supp}_b = \text{Supp}_a \cup \{x\}$ and $\text{Att}_b = \text{Att}_a$ \\ due to the consistency hypothesis $\textup{Att}_b \cap \textup{Supp}_b = \emptyset$.
    
    \begin{itemize}
		\item If $\text{Att}_b = \text{Att}_a \cup \{x\}$ and $\text{Supp}_b = \text{Supp}_a$ then: \[\varphi_{\mathcal{S}}(\{\accs(d) \mid d \in \text{Supp}_b\}) = \varphi_{\mathcal{S}}(\{\accs(d) \mid d \in \text{Supp}_a\}).\]
		
		Moreover, by definition:\begin{align*}
		    \varphi_{\mathcal{R}}(\{\accs(c) \mid c \in \text{Att}_b\}) = \varphi_{\mathcal{R}}(\{\accs(c) \mid c \in \text{Att}_a\}, \accs(x)).
		\end{align*}
		
		As $\accs(x) = 0$ and $0$ is  a neutral element of $\varphi_{\mathcal{R}}$, then: \begin{align*}
		    \varphi_{\mathcal{R}}(\{\accs(c) \mid c \in \text{Att}_a\}, \accs(x)) &=\varphi_{\mathcal{R}}(\{\accs(c) \mid c \in \text{Att}_a\}, 0) \\
            &=\varphi_{\mathcal{R}}(\{\accs(c) \mid c \in \text{Att}_a\}).
		\end{align*}

        \item Similarly, if $\text{Supp}_b = \text{Supp}_a \cup \{x\}$, $\text{Att}_b = \text{Att}_a$ and $0$ is a neutral element for $\varphi_\mathcal{S}$ then: \begin{align*}
		    &\varphi_{\mathcal{S}}(\{\accs(d) \mid d \in \text{Supp}_b\}) = \varphi_{\mathcal{S}}(\{\accs(d) \mid d \in \text{Supp}_a\}) \\
            \textup{and }&\varphi_{\mathcal{R}}(\{\accs(c) \mid c \in \text{Att}_b\}) =\varphi_{\mathcal{R}}(\{\accs(c) \mid c \in \text{Att}_a\}).
		\end{align*}   
	\end{itemize}

In both cases, if $0$ is a neutral element of $\varphi_{\mathcal{R}}$ and $\varphi_{\mathcal{S}}$, then $\accs(a) = \accs(b)$.
\end{proof}

\begin{proposition*}[Reinforcement]
	If $\varphi_{\mathcal{R}}$ and $\varphi_{\mathcal{S}}$ are increasing (resp. strictly increasing) and $\varphi_f$ is decreasing in $x$ and increasing in $y$ (resp. strictly monotonous in $x$ and $y$), $\accs$ satisfies the reinforcement axiom (resp. strict reinforcement)~(\textbf{A8}).
\end{proposition*}

\begin{proof}
    Let $A \in \textup{ac-WAG}$ be an acyclic QBAF, and $\accs$ an aggregative semantics such that $\varphi_{\mathcal{R}}$ and $\varphi_{\mathcal{S}}$ are increasing (resp. strictly increasing) and $\varphi_f$ is decreasing in $x$ and increasing in $y$ (resp. strictly monotonous in $x$ and $y$).
	
	Let $a,b \in \mathcal{A}$ be arguments, $C,C' \subseteq \mathcal{A}$ be subsets of arguments and $x,x',y,y' \in \mathcal{A} \setminus (C\cup C')$ be arguments in these subsets such that: \begin{itemize}
		\item $w(a) = w(b) = z$
		\item $\text{Att}_a = C \cup \{x\}, \text{Att}_b = C \cup \{y\}$
		\item $\text{Supp}_a = C' \cup \{x'\}, \text{Supp}_b = C' \cup \{y'\}$
		\item $\accs(x) \leq \accs(y)$, $\accs(x') \geq \accs(y')$ 
	\end{itemize}
	
	To simplify the notations in this proof, we denote the degrees of acceptability of the arguments of a set by the set itself. Consequently, we note $\varphi_{\mathcal{R}}(C,\accs(x))$ instead of $\varphi_{\mathcal{R}}(\{\accs(c) \mid c \in C\},\accs(x))$.
		
	As $\varphi_{\mathcal{R}}$ is increasing: \[\varphi_{\mathcal{R}}(C,\accs(x)) \leq \varphi_{\mathcal{R}}(C,\accs(y)).\]
	
	Similarly, as $\varphi_{\mathcal{S}}$ is increasing: \[\varphi_{\mathcal{S}}(C',\accs(x')) \geq \varphi_{\mathcal{S}}(C',\accs(y')).\]
	
	As $\varphi_f$ is decreasing in $x$ and increasing in $y$: \begin{align*}
        \varphi_f(\varphi_{\mathcal{R}}(C,\accs(x)),\varphi_{\mathcal{S}}(C',\accs(x')),z)
        \geq~&\varphi_f(\varphi_{\mathcal{R}}(C,\accs(y)),\varphi_{\mathcal{S}}(C',\accs(y')),z).
	\end{align*}
	
	Therefore, $\accs(a) \geq \accs(b)$.
	
	The proof is the same for strict reinforcement using the strict monotony of the aggregation functions.
\end{proof}

\begin{proposition*}[Monotony]
	If $0$ is a neutral element for $\varphi_{\mathcal{R}}$ and $\varphi_{\mathcal{S}}$, and if the three aggregation functions are monotonous (resp. strictly monotonous), then  $\accs$ verifies the monotony (resp. strict monotony) axiom~(\textbf{A7}).
\end{proposition*}

\begin{proof}
	Let $a,b \in \mathcal{A}$ be two arguments such that: \begin{itemize}
	    \item $w(a) = w(b)$
        \item $\text{Att}_a \subseteq \text{Att}_b$
        \item  ${\text{Supp}_b \subseteq \text{Supp}_a}$
	\end{itemize} 
	
	We note: \begin{itemize}
	    \item $\vec{x_a} = \{x_{1}, \dotsc, x_n\} = \text{Att}_a$
        \item $\vec{x_b} = \{x_{n+1}, \dotsc, x_m\} = \text{Att}_b \setminus \text{Att}_a$
	\end{itemize}
	
	Similarly: \begin{itemize}
	    \item $\vec{y_b} = \{y_{1}, \dotsc, y_{p}\} = \text{Supp}_b$
        \item $\vec{y_a} = \{y_{p+1}, \dotsc, y_{q}\} = \text{Supp}_a \setminus \text{Supp}_b$
	\end{itemize}
	
	To simplify the notations in this proof, we again denote the degrees of acceptability of the arguments of a set by the set itself. Consequently, we note $\varphi_{\mathcal{R}}(\vec{x_a})$ instead of $\varphi_{\mathcal{R}}(\{\accs(x_{a,i}) \mid x_{a,i} \in \vec{x}_a\})$.
	
	By definition, and because $0$ is a neutral element for $\varphi_{\mathcal{R}}$ and $\varphi_{\mathcal{R}}$ is non decreasing:  \[\pi_{\mathcal{R}}(b) = \varphi_{\mathcal{R}}(\vec{x_a}, \vec{x_b}) \geq \varphi_{\mathcal{R}}(\vec{x_a}, \vec{0}) =  \varphi_{\mathcal{R}}(\vec{x_a}) = \pi_{\mathcal{R}}(a).\]
	
	Similarly with $\varphi_{\mathcal{S}}$: \[\pi_{\mathcal{S}}(a) = \varphi_{\mathcal{S}}(\vec{y_b}, \vec{y_a}) \geq \varphi_{\mathcal{S}}(\vec{y_b}, \vec{0}) =  \varphi_{\mathcal{S}}(\vec{y_b}) = \pi_{\mathcal{S}}(b).\]
	
	As $\varphi_f$ is decreasing in $x$ and increasing in $y$, we can conclude: \begin{align*}
	    \accs(b) &= \varphi_{f}(\pi_{\mathcal{R}}(b), \pi_{\mathcal{S}}(b),w(b)) \\
        &\leq \varphi_{f}(\pi_{\mathcal{R}}(a), \pi_{\mathcal{S}}(a),w(a)) = \accs(a).
	\end{align*}
	
	The proof is the same for strict monotony axiom using the strict monotony of the aggregation functions.
	
\end{proof}

\end{document}

\endinput

\begin{table*}[t!]
	\centering
    \caption{Classical properties of aggregation functions. Notations: variables written in bold and indexed by $n$, such as $\mathbf{x}_n= (x_1,\dotsc x_n)$,  are elements of $[0,1]^n$, the other variables are elements of $[0, 1]$.}
    \vspace{2mm}
    
	\resizebox{0.8\linewidth}{!}{\begin{tabular}{|l|c|}
			\hline
			Name & Property of $\varphi$\\ \hline
			(P1): Boundary conditions & $\varphi(0,\dotsc, 0) = 0$ and $\varphi(1, \dotsc, 1) = 1$ \\ \hline
			(P2): Monotony (non decreasing) & 
			if $x'_{n+1} \leq x_{n+1}$, 
			$\quad \varphi(\mathbf{x}_n,x'_{n+1}) \leq \varphi(\mathbf{x}_n,x_{n+1})$ \\ \hline
			(P3): Continuity & 
			$\lim\limits_{\mathbf{x}'_n \to \mathbf{x}_n} \varphi(\mathbf{x}'_n) = \varphi(\mathbf{x}_n)$ \\ \hline
			(P4): Commutativity & 
			Let $\sigma$ be a 
			permutation on $\{1,\dotsc,n\}$, $\quad \varphi(\mathbf{x}_n) = \varphi(x_{\sigma(1)},\dotsc,x_{\sigma({n})})$ \\ \hline
			(P5): Idempotence & 
			$\varphi(x,\dotsc,x) = x$\\ \hline
			(P6): Associativity & 
			$ \varphi(\varphi(\mathbf{x}_n),\varphi (\mathbf{y}_m)) = \varphi(\mathbf{x}_n,\mathbf{y}_m)$\\ \hline
			(P7): Weakening & $\varphi(\textbf{x}_n) \leq \min(\textbf{x}_n)$ \\ \hline
			(P8): Reinforcement & $\varphi(\textbf{x}_n) \geq \max(\textbf{x}_n)$\\ \hline
			(P9): Composition & If $\varphi(\mathbf{x}_n) \leq \varphi(\mathbf{y}_m)$, \hspace{-3mm}$\quad \varphi(\mathbf{x}_n,\mathbf{z}) \leq \varphi(\mathbf{y}_m,\mathbf{z})$ \hfill for every position of $\mathbf{z}$\\ \hline
			(P10): Decomposition & 
			If $\varphi(\mathbf{x}_n,\mathbf{z}) \leq \varphi(\mathbf{y}_m,\mathbf{z})$, \hspace{-3mm}$\quad \varphi(\mathbf{x}_n) \leq \varphi(\mathbf{y}_m)$ \hfill for every position of $\mathbf{z}$\\ \hline
	\end{tabular}}	
	\label{tab:prop_agreg}
\end{table*}